\documentclass[a4paper,11pt]{article}
\usepackage{fullpage}
\makeatletter

\usepackage{times}
\usepackage{helvet}
\usepackage{courier}
\usepackage{enumerate}
\usepackage{marvosym}
\usepackage{amssymb}
\usepackage{amsmath}
\usepackage{amsthm}
\usepackage{caption}
\usepackage{subcaption}
\usepackage{graphicx}
\usepackage[ruled,vlined,noend]{algorithm2e}
\usepackage{comment}
\usepackage{url}
\usepackage{multirow}
\usepackage[usesnames,dvipsnames]{xcolor}
\usepackage[normalem]{ulem}
\usepackage{breakcites}
\graphicspath{{../img/}, {/}, {img/},{img/src/}}

\SetCommentSty{algcomment}

\newtheorem{theorem}{Theorem}
\newtheorem{definition}{Definition}
\newtheorem{lemma}{Lemma}

\newtheorem{corollary}{Corollary}

\providecommand{\tuple}[1]{\ensuremath{\langle #1 \rangle}}

\DeclareMathOperator{\append}{\oplus\;}
\providecommand{\nin}[1]{\ensuremath{\not{#1}}}

\providecommand{\Adepmax}{\ensuremath{\alpha}}
\providecommand{\Amax}{\ensuremath{\actions^{max}}}
\providecommand{\Smax}{\ensuremath{\states^{max}}}
\providecommand{\Adep}{\ensuremath{\mathcal{\actions}}}
\providecommand{\act}{\ensuremath{a}}
\providecommand{\actions}{\ensuremath{A}}

\providecommand{\agents}{\ensuremath{N}}

\providecommand{\subagents}{\ensuremath{\agents'}}
\providecommand{\crg}{\ensuremath{\phi}}
\providecommand{\cri}{\ensuremath{\mathit{CRI}}}
\providecommand{\depreward}{\ensuremath{\reward^{e}}}
\providecommand{\deprewards}{\ensuremath{\rewards^{e}}}
\providecommand{\feature}{\ensuremath{f}}
\providecommand{\features}{\ensuremath{F}}
\providecommand{\infl}{\ensuremath{I}}
\providecommand{\Imax}{\ensuremath{\beta}}
\providecommand{\jact}{\ensuremath{\vec{\act}\@ifnextchar{^}{\,}{}}}
\providecommand{\mmdp}{\ensuremath{M}}
\providecommand{\newfeatures}{\ensuremath{\hat{\features}}}
\providecommand{\newstate}{\hat{\state}}

\providecommand{\noinfl}{\ensuremath{\diamond}}

\providecommand{\policy}{\ensuremath{\pi}}
\providecommand{\reward}{\ensuremath{R}}
\providecommand{\rewards}{\ensuremath{\mathcal{\reward}}}
\providecommand{\state}{\ensuremath{s}}
\providecommand{\states}{\ensuremath{S}}
\providecommand{\tr}{\ensuremath{\tau}}
\providecommand{\transprob}{\ensuremath{T}}
\providecommand{\transitions}{\ensuremath{\mathcal{T}}}
\providecommand{\prob}{\ensuremath{Pr}}
\providecommand{\ptvalue}{\ensuremath{V}}
\providecommand{\width}{\ensuremath{w}}
\providecommand{\pad}{\ensuremath{\theta}}
\providecommand{\mdpreturn}{\ensuremath{Z}}
\providecommand{\lb}{\ensuremath{L}}
\providecommand{\ub}{\ensuremath{U}}
\providecommand{\wildcard}{\ensuremath{*}}

\providecommand{\infl}{\theta^{i}}

\newcommand{\citet}[1]{\citeauthor{#1} [\citeyear{#1}]}

\def\docnote{This article is an extended version of the paper that was published under the same
title in the Proceedings of the Thirtieth AAAI Conference on Artificial Intelligence (AAAI16), held
in Phoenix, Arizona USA on February 12-17, 2016. The most significant difference is that here a more
strict definition of dependent actions, transition influence and, consequentially, the conditional
return graphs is given. Furthermore, this version contains additional details and explanations that
did not make it into the conference paper due to the page limit.}

\providecommand{\doctitle}{Solving Transition-Independent Multi-agent MDPs with Sparse
Interactions\\(Extended version)\footnote{\docnote}}

\makeatother

\begin{document}
\title{\doctitle}
\author{Joris Scharpff$^1$, Diederik M. Roijers$^2$, Frans A.
Oliehoek$^{2,3}$,\\
Matthijs T. J. Spaan$^1$, Mathijs M. de Weerdt$^1$\\
$ $ \\
$^1$ Delft University of Technology, The Netherlands \\
$^2$ University of Amsterdam, The Netherlands \\
$^3$ University of Liverpool, United Kingdom}
\date{}
\maketitle

\begin{abstract}
\noindent In cooperative multi-agent sequential decision making under uncertainty, agents must
coordinate to find an optimal joint policy that maximises joint value. Typical algorithms exploit additive
structure in the value function, but in the fully-observable multi-agent MDP (MMDP) setting such
structure is not present. We propose a new optimal solver for transition-independent MMDPs, in which
agents can only affect their own state but their reward depends on joint transitions. We represent
these dependencies compactly in \emph{conditional return graphs (CRGs)}. Using CRGs the value of a
joint policy and the bounds on partially specified joint policies can be efficiently computed.
We propose CoRe, a novel branch-and-bound policy search algorithm building on CRGs. CoRe typically
requires less runtime than the available alternatives and finds solutions to previously unsolvable
problems.

\end{abstract}

\section{Introduction}
When cooperative teams of agents are planning in uncertain domains, they must coordinate to maximise
their (joint) team value. In several problem domains, such as traffic light control
\cite{Bakker10icis}, system monitoring \cite{Guestrin02}, multi-robot planning \cite{Messias13aaai}
or maintenance planning \cite{Scharpff2013mpp}, the full state of the environment is assumed to be
known to each agent. Such \emph{centralised} planning problems can be formalised as multi-agent
Markov decision processes (MMDPs) \cite{Boutilier1996mmdp}, in which the availability of complete
and perfect information leads to highly-coordinated policies. However, these models suffer from
exponential joint action spaces as well as a state that is typically exponential in the number of agents.

In problem domains with local observations, sub-classes of \emph{decentralised} models exist that
admit a value function that is exactly factored into additive components
\cite{Becker2003AAMAS,Nair2005ndpomdp,witwicki2010influence} and more general classes admit
upper bounds on the value function that are factored \cite{Oliehoek15ijcai}. In centralised models
however, the possibility of a factored value function can be ruled out in general: by observing the
full state, agents can predict the actions of others better than when only observing a local state. 
This in turn means that each agent's action should be conditioned on the full state and that 
the value function therefore also depends on the full state.

A class of problems that exhibits particular structure is that of task-based planning problems, such
as the \emph{maintenance planning problem} (MPP) from \cite{Scharpff2013mpp}. In the MPP every agent
needs to plan and complete its own set of road maintenance tasks at minimal (private) maintenance
cost. Each task is performed only once and may delay with a known probability.
As maintenance causes disruption to traffic, agents are collectively fined relative to the
(super-additive) hindrance from their \emph{joint actions}. Although agents plan autonomously, they
depend on others via these fines and must therefore coordinate. Still, such \emph{reward
interactions} are typically sparse: they apply only to certain combinations of maintenance tasks,
e.g.\ in the same area, and often involve only a few agents. Moreover, when an agent has performed
its maintenance tasks that potentially interfere with others, it will no longer interact with any of
the other agents.

Our main goal is to identify and exploit such structure in centralised models, for which we consider
\emph{transition independent} MMDPs (TI-MMDPs). In TI-MMDPS, agent rewards depend on joint states
and actions, but transition probabilities are individual. Our key insight is that we can exploit the
reward structure of TI-MMDPs by decomposing the \emph{returns} of all execution histories
(i.e., all possible state/action sequences from the initial time step to the
planning horizon) into components that depend on local states and actions.

We build on three key observations.
1) Contrary to the optimal value function, returns \emph{can} be decomposed without loss of
optimality, as they depend only on local states and actions of execution sequences.
This allows a compact representation of rewards and efficiently computable bounds on the optimal
policy value via a data structure we call the \emph{conditional return graph} (CRG).
2) In TI-MMDPs agent interactions are often sparse and/or local, for instance in the domains
mentioned initially, typically resulting in very compact CRGs.
3) In many (e.g.\ task-modelling) problems the state space is transient, i.e., states can only be
visited once, leading to a directed, acyclic transition graph. With our first two key
observations this often gives rise to \emph{conditional reward independence}, i.e.\ the absence
of further reward interactions, and enables agent decoupling during policy search.

Here we propose \emph{conditional return policy search} (CoRe), a branch-and-bound policy search
algorithm for TI-MMDPs employing CRGs, and show that it is effective when reward interactions
between agents are sparse. We evaluate CoRe on instances of the aforementioned MPP with uncertain
outcomes and very large state spaces. We demonstrate that CoRe evaluates only a fraction of the
policy search space and thus finds optimal policies for previously unsolvable instances and commonly
requires less runtime than its alternatives.

\section{Related work}
\label{sec:related-work}

Scalability is a major issue in multi-agent planning under
uncertainty. In response to this challenge, two important lines of
work have been developed. One line of work proposed
approximate solutions by imposing and exploiting an additive structure in
the value function \cite{Guestrin02}. This approach has been applied in a range of stochastic
planning settings, fully and partially observable alike, both from a single-agent perspective
 \cite{koller1999computing,parr1998flexible} and multi-agent
\cite{Guestrin2002context,meuleau1998solving,SparseCoopQ,Oliehoek13AAMAS}. The drawback of
such methods is that typically no bounds on the efficiency loss can be
given. We focus on optimal solutions, required to deal with strategic behaviour in a
mechanism~\cite{cavallo2012optimal,Scharpff2013mpp}.

This is part of another line of work that has not sacrificed optimality, but instead targets
sub-classes of problems with properties that can be exploited
\cite{Becker2003AAMAS,becker2004decentralized,mostafa2009offline,witwicki2010influence}. In
particular, several methods that exploit the same type of additive structure in the value
function have been shown exact, simply because value functions of the sub-class of
problems they address are guaranteed to have such shape
\cite{Nair2005ndpomdp,Oliehoek08AAMAS,Varakantham2007ndpomdp}. However, all these approaches are for
decentralised models in which actions are conditioned only on \emph{local} observations.
Consequentially, optimal policies for decentralised models typically yield lower value than
the optimal policies for their fully-observable counterparts (shown in our experiments).

Our focus is on transition-independent problems, suitable for multi-agent problems in which the
effects of activities of agents are (assumed) independent. In domains where agents directly
influence each other, e.g., by manipulating shared state variables, this assumption is violated.
Still, transition independence allows agent coordination at a task level, as in the MPP, and is both
practically relevant and not uncommon in literature
\cite{Becker2003AAMAS,Spaan06aamas,Melo11,Dibangoye13}.

Another type of interaction between agents is through limited (global) resources required for
certain actions. While this introduces a global coupling, some scalability is achievable
\cite{meuleau1998solving}. Whether context-specific and conditional agent independence remains
exploitable in the presence of such resources in TI-MMDPs is yet unclear. Additionally, there
exist also methods that exploit reward sparsity and independence but through a reinforcement learning
approach identifying `interaction states' \cite{Hauwere2012solving,Melo2009learning}. Although these
target a similar structure, learning implies that there are no guarantees on the solution quality
(until all states have been recognised as interaction states, which implies a brute-force solve of
the MMDP).

\section{Model} 

We consider a (fully-observable) \emph{transition-independent}, \emph{multi-agent} \emph{Markov
decision process}, or \emph{TI-MMDP}, with a finite horizon of length~$h$, and no discounting of
rewards.
\begin{definition}
A TI-MMDP is a tuple $\tuple{\agents, \states, \actions, \transprob, \rewards}$:
\begin{itemize}
  \def\itemcorr{\hspace{-3mm}}
	\item[] \itemcorr  $\agents = \{ 1, ..., n\}$ is a set of $n$ enumerated agents;  
	\item[] \itemcorr $\states = \states^1 \times ... \times \states^n$ is the agent-factored state
	space, which is the Cartesian product of $n$ factored states spaces $\states^i$
	(composed of features~$\feature \in \features$, i.e.\ $\state^i = \{ \feature^{i}_x, \feature^i_y,
	\ldots \} $);
	\item[]  \itemcorr $\actions = \actions^1 \times ... \times \actions^n$ is the joint action
	space, which is the Cartesian product of the $n$ local action spaces $\actions^i$;
	\item[] \itemcorr $\transprob( \state, \jact, \newstate ) = \prod_{i \in \agents}
	\transprob^i( \state^i, \act^i, \newstate^i)$ defines a transition probability, which is the product of
	the local transition probabilities due to transition independence; and
	\item[] \itemcorr $\rewards$ is the set of reward functions over transitions that we assume
	without loss of generality is structured as $\left\{ \depreward | e \subseteq \agents \right\}$. 
	When $e = \{ i \}$,  $\reward^i$ is the \emph{local reward} function for agent~$i$, and when $|e|>1$, $\depreward$ is called an \emph{interaction reward}.
	The total team reward per time step, given a joint state~$\state$, joint action~$\jact$ and
	new joint state~$\newstate$, is the sum of all rewards:
	\begin{equation}
	\reward(\state, \jact, \newstate) = \sum_{\depreward \in \rewards} \depreward( \{ \state^{j} \}_{j
	\in e}, \{ \jact^{j} \}_{j \in e}, \{ \newstate^j \}_{ j \in e } ).
	\end{equation}
\end{itemize}
\end{definition}

Two agents $i$ and $j$ are called \emph{dependent} when there exists a reward function with both
agents in its scope, e.g., a two-agent reward
$\reward^{i,j}(\{\state^i,\state^j\},\{\act^i,\act^j\}, \{\newstate^i, \newstate^j\})$ could
describe the super-additive hindrance that results when agents in the MPP do concurrent maintenance on two
nearby roads. We focus on problems with \emph{sparse interaction rewards}, i.e., reward
functions~$\depreward$ with non-zero rewards for a small subset of the local joint actions (e.g.,
$\Adep^{ij} \subset \actions^i \times \actions^j$) or only a few agents in its scope.
Of course, sparseness is not a binary property: the maximal number of actions with non-zero
interaction rewards and participating agents (respectively $\alpha$ and $\width$ in
Theorem~\ref{thm:space}) determine the level of sparsity. Note that this is not a restriction but
rather a classification of problems that benefit most from our approach.

The goal in a TI-MMDP is to find the optimal joint policy~$\policy^*$ of which the actions $\jact$
maximise the expected sum of rewards, expressed by the Bellman equation:
\begin{equation}\label{eq:optval} 
\ptvalue^{*}(\state_t) = \max_{\jact_t} \sum_{\state_{t+1} \in \states}
	\transprob(\state_t, \jact_t, \state_{t+1} ) \Big(
\sum_{\depreward\in\rewards} 
	\reward^e(\state_{t}^e, \jact_{t}^e,\state_{t+1}^e)+
	\ptvalue^{*}(\state_{t+1} )\Big).
\end{equation}

At the last timestep there are no future rewards, so $\ptvalue^*(\state_h) = 0$ for every~$\state_h
\in \states$. Although $\ptvalue^*(\state_t)$ can be computed through a series of maximisations over
the planning period, e.g.\ via dynamic programming \cite{puterman2014markov}, it cannot be written
as a sum of independent local value functions without losing optimality \cite{koller1999computing},\
i.e.\ $\ptvalue^*(\state) \not=\sum_{e} \ptvalue^{e,*}(\state^e)$.

Instead, we factor the \emph{returns} of \emph{execution sequences}, the sum of rewards obtained
from following state/action sequences, which is optimality preserving. We denote an execution
sequence up until time $t$ as $\pad_t = [\state_0,\jact_0,$ $...,\state_{t-1}, \jact_{t-1}, \state_t
]$ and its return is the sum of its rewards: $\sum_{x=0}^{t-1}
\reward(\state_{\pad,x},\jact_{\pad,x}, \state_{\pad,x + 1})$, where $\state_{\pad,x}$,
$\jact_{\pad,x}$ and $\state_{\pad,x + 1}$ respectively denote the state and joint action at
time~$x$, and the resulting state at time~$x + 1$ in this sequence. A seemingly trivial but
important observation is that the return of an execution sequence can be written as the sum of local
functions:
\begin{equation}\label{eq:returnDistr}
	\mdpreturn (\pad_t) = \sum_{\depreward \in \rewards}  \sum_{x=0}^{t-1}  \reward^e(\state_{\pad,x}^e,
	\jact_{\pad,x}^e,\state_{\pad,x+1}^e),
\end{equation}
where $\state_{\pad,x}^e$, $\jact_{\pad,x}^e$ and $\state_{\pad,x+1}^e$ denote local
states and actions from $\pad_t$ that are relevant for $\reward^e$. Contrary to the optimal value
function, \eqref{eq:returnDistr} is additive in the reward
components and can thus be computed locally.

Nonetheless, the return of \eqref{eq:returnDistr} does not directly give us the value of an
optimal policy. To compute the expected policy value using \eqref{eq:returnDistr}, we sum the
expected return of all future execution sequences~$\pad_h$ reachable under policy~$\policy$
starting at $\state_0$ (denoted $\pad_h | \policy, \state_0$):
\begin{align}
\ptvalue^\policy( \state_0 )
&= \sum_{\pad_h | \policy, \state_0 } \prob( \pad_h ) \mdpreturn( \pad_h )
= \sum_{\pad_h | \policy, \state_0 } \mdpreturn( \pad_h ) \prod_{t=0}^{h - 1} \transprob(
\state_{\pad, t}, \policy( \state_{\pad, t} ),  \state_{\pad, t+1} ) .
	\label{eq:optreturn}
\end{align}

Now, \eqref{eq:optreturn} is structured such that it expresses the value in terms of additively
factored terms ($\mdpreturn( \pad_h ) $). However, comparing \eqref{eq:optval} and
\eqref{eq:optreturn}, we see that the price for this is that we no longer are expressing the
\emph{optimal} value function, but that of a given policy~$\policy$. In fact, \eqref{eq:optreturn}
corresponds to an equation for \emph{policy evaluation}. It is thus not a basis for dynamic
programming, but it is usable for policy search. Although policy search methods have their own
problems in scaling to large problems, we show that the structure of \eqref{eq:optreturn} can be
leveraged.

In particular, because the return of an execution sequence can be decomposed into additive
components (Eq.~\ref{eq:returnDistr}), we decouple and store these returns locally in
\emph{conditional return graphs} (CRGs). The CRG is used in policy search to efficiently compute
\eqref{eq:optreturn} when, during evaluation, the transition probability~$\prob(\pad)$ of an
execution sequence~$\pad$ becomes known. Moreover, the returns stored in the CRG can be used to
bound the expected value of sequences, allowing branch-and-bound pruning. Finally, when constructing
CRGs it is possible to detect the absence future reward interactions between agents and thus
optimally decouple the policy search.

\section{Conditional Return Graphs}
We now partition the reward function into additive components~$\rewards_i$ and assign them to
agents. The \emph{local} reward for an agent~$i \in \agents$  is given by $\rewards_i =
\{\reward^i\} \cup \rewards_i^e$, where $\rewards_i^e$ are the interaction rewards assigned to $i$
(restricted to $\depreward$ where  $i\in e$). The sets $\rewards_i$ are disjoint sub-sets of the
reward functions~$\rewards$ such that together they again form the complete set of joint reward
functions. Note that such a partitioning can be done in many ways. In our preliminary experiments we
observed that for the maintenance planning domain balancing the number of functions per agent works
well. Nevertheless, further study is required to establish potentially better or more generic
assignment heuristics.

Given such a disjoint partition of rewards, a conditional return graph for agent~$i$ is a data
structure that represents all possible \emph{local} returns, for all possible \emph{local} execution
histories. Particularly, it is a directed acyclic graph (DAG) with a layer for every stage
$t=0,\dots,h-1$ of the decision process. Each layer contains nodes corresponding to the reachable
\emph{local states} $\state^i \in \states^i$ of agent~$i$ at that stage. As the goal is to include
interaction rewards, the CRG includes for every local state $\state^i$, local action $\act^i$, and 
successor state $\newstate^i$ a representation of all transitions $(\state^e, \jact^e, \newstate^e)$
for which $\state^i \in \state^e$, $\act_i\in\jact^e $, and $\newstate^i \in \newstate^e$.

While a direct representation of these transitions captures all the rewards possible, we can achieve
a much more compact representation by exploiting sparse interaction rewards, enabling us to group
many joint actions $\jact^e$ leading to the same rewards. Consider a two-agent example $\agents =
\{1, 2\}$ with actions~$\actions^1 = \{ a^1 \}$ and $\actions^2 = \{ a^2, b^2, c^2 \}$ respectively.
Both agents have a local reward function, resp.\ $\reward^1$ and $\reward^2$, and there is one
interaction reward~$\reward^{1,2}$ that is $0$ for all transitions but the ones involving joint
action~$\{a^1, a^2\}$. This could for example be a network cost function of MPP that is non-zero
when the maintenance activities corresponding to actions~$a^1$ and $a^2$ are performed, e.g.\
because they take place within close proximity of each other. The interaction reward~$\reward^{1,2}$
is assigned to agent~$1$, thus $\rewards_i^1 = \{ \reward^i, \reward^{1,2} \}$ and of course
$\rewards_i^2 = \{ \reward^2 \}$. A naive representation of all rewards would result in the graph of
Figure~\ref{fig:crgdag-1}, illustrating a transition from $\state^1_0$ to $\state^1_1$ where
agent~$2$ may remain in state~$\state^2_4$ or transition to state~$\state^2_6$.

\begin{figure}[t]
	\centering
	\def\figwidth{0.48\columnwidth}
	\begin{subfigure}[b]{\figwidth}		
		\includegraphics[width=\textwidth]{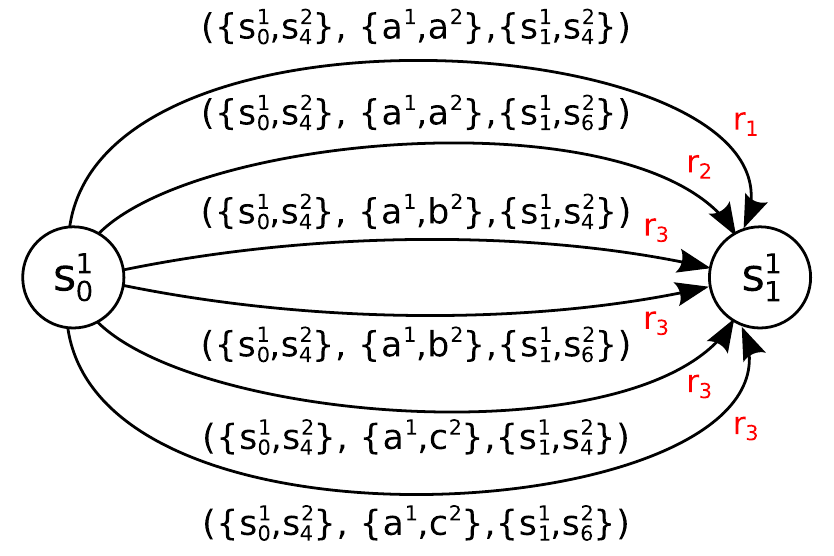}
		\caption{~}
		\label{fig:crgdag-1}
	\end{subfigure}
	\hspace{4mm}
	\begin{subfigure}[b]{\figwidth}		
		\includegraphics[width=\textwidth]{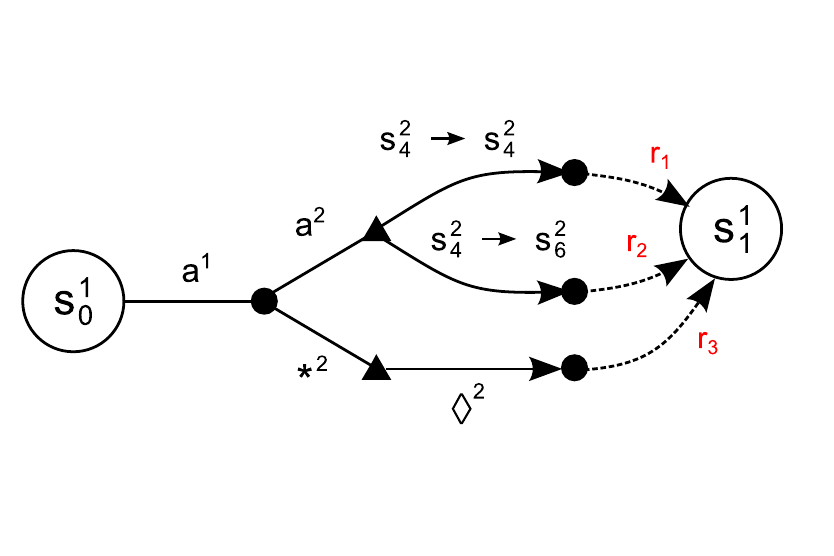}
		\caption{~}
		\label{fig:crgdag-2}
	\end{subfigure}
	\caption{Example of a transition for one agent of a two-agent problem where            
            (\subref{fig:crgdag-1}) shows the complete state/transition graph with unique
            rewards~$r_x$ and (\subref{fig:crgdag-2}) the equivalent but more compact CRG
            when~$\reward_1$ only depends on~$\act^2_1$.}
	\label{fig:crgdag}
\end{figure}

Observe that now there are only three unique rewards (shown in red) whereas there are six possible
transitions. Intuitively, all transitions resulting in the same reward should be grouped such that
only when the actions/states of other agents influence the reward they should be included in the
CRG, resulting in the graph of Figure~\ref{fig:crgdag-2}. To make this explicit, we denote a single
transition for agents~$e \subseteq \agents$ by $\tr^e = (\{ \state^{j} \}_{j \in e},$ $\{ \jact^{j}
\}_{j \in e}, \{ \newstate^j \}_{ j \in e }) = (\state^e, \jact^e, \newstate^e)$. A local
transition~$\tr^i$ is said to be contained in $\tr^e$, denoted $\tr^i \in \tr^e$, if $i \in e$,
$\state^i \in \state^e$, $\act^i \in \jact^e$ and $\newstate^i \in \newstate^e$. Moreover the set
set of all available (joint) transitions can be written as

\begin{equation}
	\transitions^e = \{ (\state^e, \jact^e, \newstate^e ) \, | \, \state^e, \newstate^e \in \{
	\state^{j} \}_{j \in e}, \jact^e \in \{ \jact^{j} \}_{j \in e}, \transprob( \state^e, \jact^e,
	\state^e ) > 0 \}
\end{equation}

Now we can formalise the set of actions (state transitions follow briefly afterwards) of other
agents that may interact with the rewards~$\rewards_i$, assigned to each of the CRGs, given a
current local transition~$\tr^i = (\state^i, \act^i, \newstate^i)$. An action~$\act^j$ of agent~$j
\neq i$ is said to be \emph{dependent} with respect to local transition~$\tr^i$ if it occurs in one
of the available joint transitions that contains~$\tr^i$, its presence influences the interaction
reward and there is at least one other action of agent~$j$ that does \emph{not} cause the same
interaction reward. The last condition is included to prevent marking all actions as dependent when
actually the interaction reward depends on the state transition of agent~$j$. This leads to the
following definition of \emph{dependent actions}:

\begin{definition}[Dependent Actions]
	The set of \emph{dependent actions} of an agent~$j \in
\agents$ that may reward-interact with agent~$i \neq j$ when agent~$i$'s local transition is $\tr^i
= (\state^i, \act^i, \newstate^i)$ is given by:
	\begin{align*}
		\Adep( \tr^i, j) = \{ \act^j \in \actions^j \, | & \, \exists \, \depreward \in
		\rewards_i, \, \exists \, \tr^e \in \transitions^{e}, \exists \, b^j
		\neq \act^j \in \actions^j \!: \\
		&~~ \tr^i \in \tr^e \land \act^j \in \jact^{e} \\
		&~~ \land \depreward( \tr^e ) \neq \depreward( \tr^e \setminus \{ \tr^j \} )
		, \hspace{30mm} \text{s.t.\ } \tr^j \in \tr^e \\
		&~~ \land \depreward( \tr^e ) \neq \depreward( \state^{e}, \jact^{e} \setminus
		\{\act^j\} \cup \{ b^j \}, \newstate^{e} ) \}
	\end{align*}
	\label{def:dependent-actions}
\end{definition}

Actions by other agents that are not dependent with respect to a transition~$\tr^i$, i.e.\ the
actions $\actions^j \setminus \Adep( \tr^i, j )$, are (made) anonymous in the CRG for agent~$i$
via `wildcards' (e.g.\ $\wildcard^2$ of Figure~\ref{fig:crgdag-2}), since they do not influence the
reward from the functions in $\reward^i$.

Besides actions, the interaction reward may also be affected by the state transitions of the other
agents. This is captured by the \emph{transition influence}.\footnote{In~\cite{Scharpff2013mpp} only
the set of dependent actions was defined. Although the definition in the paper is correct, for some
problems it may lead to an unnecessary blow-up of the CRG size. For instance, an interaction reward
assigned to agent~$i$ may actually depend on the state transition of another agent~$j$, regardless
of the action it performs. By the previous definition all actions of agent~$i$ would have been
marked dependent. Here we strengthen the previous definition by decoupling the influence of actions
and state transitions on the interaction reward, such that the CRGs constructed based on these
definitions are indeed of minimum size.} Its definition is rather similar to that of dependent
actions. For a state transition $\state^j \rightarrow \newstate^j$ of an agent~$j \neq i$ to be
considered an influence with respect to local transition~$\tr^i$, both states must be part of a
transition~$\tr^j$ such that both $\tr^i$ and $\tr^j$ are contained in a joint transition~$\tr^e$,
such a transition~$\tr^j$ must have an impact on at least one interaction reward and there must
exist at least one other transition of agent~$j$ that does not have the same interaction reward
impact. The latter condition is, as before, to prevent all state transitions being marked an
influence whereas the interaction reward depends solely on the action. This is formalised by the
following definition:

\begin{definition}
The set of state pairs of an agent~$j$ that may lead to a
reward interaction when agent~$i \neq j$ perform transition~$\tr^i = (\state^i, \act^i, \newstate^i)$
and agent~$j$ performs action~$\act_j \in \actions_j$ concurrently, is known as the
\emph{(transition) influence}, defined as
	\begin{align}
		\infl^i( \tr^i, \act^j )	= \{ (\state^j, &\newstate^j) \in
		\states^j \times \states^j \, | \, \exists \, \depreward \in \rewards_i, \exists \, \tr^e \in
		\transitions^e, \exists \, (\state_2^j, \newstate_2^j) \neq (\state^j, \newstate^j) \in
		\states^j \times \states^j \!:
		\nonumber \\
		& ~~ \tr^i \in \tr^e \land \tr^j = (\state^j, \act^j, \newstate^j) \in \tr^{e} \\
		& ~~ \land \depreward( \tr^e ) \neq \depreward( \tr^e \setminus \{ \tr^j \} )
		, \hspace{40mm} \text{s.t.\ } \tr^j \in \tr^e \\
		& ~~ \land  \depreward( \tr^e ) \neq \depreward( \state^e \setminus \{ \state^j \} \cup
		\{ \state_2^j \}, \jact^e, \newstate^e \setminus \{ \newstate^j \} \cup
		 \{ \newstate_2^j \} )
		\label{eq:influence-other}
		\}
	\end{align}
	\label{def:influence}
\end{definition}

Finally, for any set of actions~$\actions^j$ of an agent~$j$ we define the transition influence of
that set with respect to local transition~$\tr^i$ as the union of all influences, or $\infl^i(
\tr^i, \actions^j ) = \bigcup_{\act^j \in \actions^j } \infl^i( \tr^i, \act^j )$. This last
definition is useful to capture the influence of a wildcard set~$\wildcard^j$, which occurs when
multiple actions lead to the same state-transition interaction reward. Again, non-influencing state
transitions can be grouped. We use the symbol~$\noinfl^j$ to denote the set of all non-influencing
transitions of agent~$j$ given a local transition~$\tr^i$ and action~$\act^j$ (or
wildcard~$\wildcard^j$).

Given Defs.~\ref{def:dependent-actions} and \ref{def:reward-indep} above, a \emph{conditional return
graph} (CRG)~$\crg_i$ for agent~$i$ can be defined as follows.

\begin{definition}[Conditional Return Graph] \def\anode{v} \def\inode{u} Given a disjoint, complete
partitioning~$\rewards = \bigcup_{i \in \agents} \rewards_i$ of rewards over agents~$i \in \agents$,
the \emph{Conditional Return Graph (CRG)}~$\crg^i$ is a \emph{directed acyclic graph} with for every
stage $t$ of the decision process a node for every reachable local state~$\state^i \in \states$ and
for every available local transition~$\tr^i = (\state^i, \act^i, \newstate^i)$ a tree compactly
representing all joint transitions~$\tr^e = (\state^e, \jact^e, \newstate^e)$ of the agents~$e
\subseteq \agents$ in the scope of $\rewards_i$, or $e = \{ i \in \agents \, | \, \exists \depreward
\in \rewards_i\!: i \in e \}$.
 
The tree consists of two parts: an \emph{action tree} that specifies all dependent actions and an
\emph{influence tree} that contains the relevant local state transitions. For every action~$\act^i
\in \actions^i$ of agent~$i$, the state $\state^i$ is connected to the root node $\anode_{\act^i}$ of
an action tree by an arc labeled with action $\act^i$. For every root node $\anode_{\act^i}$, let
$\anode = \anode_{\act^i}$ be the root of an action tree such that it is defined recursively over the
remaining $\agents' = e \setminus \{ i \}$ agents as follows:

\begin{enumerate}[{A}1] \item If $\agents' \neq emptyset$ take some $j \in \agents'$, otherwise
stop.
\item For every $\act^j \in \Adep( \tr^i,  j)$, create an internal node~$\anode_{\act^j}$
connected from $\anode$ by an arc labeled with the action $\act^j$.
\item Create one internal node~$\anode_{\wildcard^j}$ to represent all actions of agent~$j$ not in
$\Adep( \tr^i, j)$ (if any), connected by an arc labeled by the `other action'
wildcard~$\wildcard^j$ from the root node~$\anode$.
\item For each leaf node~$\anode_{\act^j}$ (or $\anode_{\wildcard^j}$) that results from the previous
steps, create a sub-tree with $N'=N'\setminus\{j\}$ and $\anode = \anode_{\act^j}$ as its root using
the same procedure.
\end{enumerate}

When all action arcs have been created, each leaf node~$\anode_{\act^x}$ of the action tree is the
root node~$\inode$ of an influence tree, where $\act^x$ is either the last dependent action or
wildcard~$\wildcard^x$ of the agent~$x$ that is visited in the last recursion. Starting again from
$\agents' = e \setminus \{ i \}$:

\begin{enumerate}[{B}1] \item If $\agents' \neq emptyset$ take some $j \in \agents'$, otherwise
stop.
\item If the path from $\state^i$ to node~$\inode$ contains an arc labelled with action $\act^j \in
\Adep( \tr^i, j)$, create child nodes~$\inode_{\state^j \rightarrow \newstate^j}$ to represent all
local pairs of state transitions $(\state^j, \newstate^j)$ of agent~$j$ in the action
influence~$\infl^i( \tr^i, \act^j )$, connected to node~$\inode$ by arcs labeled $\state^j \rightarrow
\newstate^j$.

\textbf{else}

The path from $\state^i$ to node~$\inode$ contains the wildcard~$\wildcard^j$. Create child
nodes~$\inode_{\state^j \rightarrow \newstate^j}$ for all pairs of local state transitions~$(\state^j,
\newstate^j) \in \infl^i( \tr^i, \wildcard^j )$, i.e.\ the influence of the set of actions
represented by~$\wildcard^j$ (all $\act^j \notin \Adep( \tr^i, j)$), and connect them to
$\inode$ with arcs labelled $\state^j \rightarrow \newstate^j$.
\label{def:crg-infl-set}

\item If there remains any pair of local states $(\state^j, \newstate^j) \in \states^j \times
\states^j$ with $\transprob( \state^j, \act^j, \newstate^j ) > 0$ that is not in $\infl^i( \tr^i,
\act^j )$ or a pair with $\sum_{\act^j \in \wildcard^j} \transprob( \state^j, \act^j, \newstate^j )
> 0$ that is not in $\infl^i( \tr^i, \wildcard^j )$, depending on the action of agent~$j$ on the
path to node~$\inode$, create another child node~$\inode_{\noinfl^j}$ connected by an arc labeled by
the `other state pairs' wildcard~$\noinfl^j$.

\item For each leaf node~$\inode_{\state^j \rightarrow \newstate^j}$ (or $\inode_{\noinfl^j}$) that
results from the previous step, create a sub-tree with $\agents' = \agents' \setminus \{ j \}$ and
root~$\inode = \inode_{\state^j \rightarrow \newstate^j}$ (resp.\ $\inode = \inode_{\noinfl^j}$)
repeating the same procedure.
\end{enumerate}

 Finally, each leaf node~$\inode_{\state^x \rightarrow \newstate^x}$ ($x$ again being the last agent) of
 every influence tree is connected to the new local state node~$\newstate^i$ by an arc labeled with
 the transition reward~$\rewards_i( \tr^e )$ that corresponds to the actions and state pairs on
 the path from~$\state^i$ to $\newstate^i$.
 \label{def:crg}
 \end{definition}

The labels on the path to a leaf node of an influence tree, via a leaf node of the action tree,
sufficiently specify the joint transitions of the agents in scope of the functions $\reward^e \in
\rewards^i$, such that we can compute the reward $\sum_{\depreward\in\rewards_i}
\depreward(\state^e, \jact^e,\newstate^e)$. Note that for each $\depreward$ for which an action is
chosen that is not in $\Adep(\act^i,j)$ (a wildcard~$\wildcard^j$ in the action tree), the
interaction reward must be $0$ by definition (and similarly for state transitions in~$\noinfl^j$).

In Figure~\ref{fig:crgdag-2} an example CRG is illustrated. The local state nodes are displayed as
circles; the internal nodes as black dots and action tree leaves as black triangles. The action arcs
are labelled $a^1$, $a^2$ and `wildcard' $*^2$, whereas transition influence arcs are labelled
$(s^2_4 \rightarrow s^2_4)$, $(s^2_4 \rightarrow s^2_6)$ and $\noinfl^j$. Note that
Definition~\ref{def:crg} captures the general case, but often it suffices to consider
transitions~$(\state^i \cup \features^{e \setminus i}, \jact^e, \newstate^i \cup \newfeatures^{e
\setminus i})$, where $\features^{e \setminus i}$ is the set of state features on which the reward
functions $\rewards_i$ depend. This is a further abstraction: only feature influence arcs are
needed, typically requiring much less arcs (demonstrated later in Figure~\ref{fig:CRG}).

Now we investigate the maximal size of the CRGs to derive a theoretical upper bound.
Let $|\Smax| = \max_{i \in \agents} |\states^i|$ and $|\Amax| = \max_{i \in \agents} |\actions^i|$
denote respectively the maximal individual state and action space sizes, $\width = \max_{\depreward
\in \deprewards} |e| - 1$ denote the maximal interaction function scope size, $\Adepmax = \max_{i,j
\in \agents} \max_{\act^i \in \actions^i} |\Adep(\act^i, j)|$ the set size of the largest dependent
action set, and let finally $\Imax = \max_{i,j \in \agents} \max_{\tr^i \in \transitions^i}
\max_{\act^j \in \actions^j } | \infl( \tr^i, \act^j ) |$ denote the size of the largest transition
influence set. First note that the full joint policy search space is $\Theta( h |\Smax|^{2n}
|\Amax|^{n} )$, however we show that the use of CRGs can greatly reduce this:
\begin{theorem}
The maximal size of a CRG is 
\begin{equation}
O( ~ h \cdot |\Amax| |\Smax| ^ {2} \cdot (\Adepmax\Imax)^{\width} ~ ).
\label{eq:space}
\end{equation}
\vspace{-5mm}
\label{thm:space}
\end{theorem}
\begin{proof}
A CRG has as many layers as the planning horizon~$h$. In the worst case, in every stage there are 
$|\Smax|$ local state nodes, each connected to at most $|\Smax|$ next-stage local state nodes via
multiple arcs. The number of action arcs between two local state nodes~$\state^i$ and $\newstate^i$
is at most $|\actions^i|$ times the maximal number of dependent actions, $\Adepmax^\width$.
Finally, the number of influence arcs is bounded by $\Imax^\width$.
\end{proof}

Note that in general all actions and transitions can cause interaction rewards, in which case the
size of all $n$ CRGs combined is $O(n h |\Smax|^{2+2 \width} |\Amax|^{1+\width} )$; typically still
much more compact than the full joint policy search space unless $\width \approx |\agents|$. For
many problems however, the interaction rewards are more sparse and
$\Adepmax^{\width}\!\ll\!|\Amax|^{\width}$. Moreover, (\ref{eq:space}) gives an upper bound on the
CRG size in general, for a specific CRG~$\crg_i$ this bound is often expressed more tightly by $O( h
\cdot |\actions^i| |\states^i| ^ {2} \cdot \prod_{j \in \agents} (\Adepmax^i \Imax^i)^\width )$,
where $\Adepmax^i$ and $\Imax^i$ denote respectively the maximal dependent action and transition
influence set sizes for agent~$i$. Finally, each $|\states^i|$ can be reduced to $|\features^i|$
when conditioning on state features is sufficient.

\paragraph{Bounding the optimal value}
In addition to storing rewards compactly, we use CRGs to bound the optimal policy value.
Specifically, the maximal (resp. minimal) return from a joint state $\state_t$ onwards, is an upper
(resp. lower) bound on the attainable reward, later to combined with its probability to obtain the
expected value. Moreover, the sum of bounds on local returns bounds the global return and thus on
the globally optimal joint policy value. We define the bounds recursively:
\begin{equation}
	\ub( \state^i ) =	\max\limits_{ (\state^e, \jact^e_t, \newstate^e ) \in
	\crg_i(\state^i)} \left( \rewards_i( \state^e, \jact^e_t, \newstate^e ) + \ub( \newstate^i
	)\right),
	\label{eq:UB}
\end{equation}
such that $\crg_i( \state^i )$ denotes the set of local transitions available from state~$\state^i
\in \state^e$ (ending in $\newstate^i \in \newstate^e$) represented in CRG~$\crg_i$. The bound
on the optimal value for a joint transition $(\state, \jact, \newstate)$ of all agents is
\begin{equation}
	\ub ( \state, \jact_t, \newstate) = \sum_{i \in N} \left( \rewards_i( \state^e, \jact^e_t,
	\newstate^e ) + \ub( \newstate^i ) \right),
	\label{eq:boundTrans}
\end{equation}
and lower bound~$L$ is defined similarly over minimal returns.

\paragraph{Conditional Reward Independence}
Furthermore, CRGs can exploit independence in local reward functions as a result of past
decisions. In many task-modelling MMDPs, e.g.\ those mentioned in the introduction, actions can be
performed a limited amount of times, after which reward interactions involving that action no longer occur. When an agent can no longer
perform dependent actions, the expected value of the remaining decisions is found
through local optimisation. More generally, when dependencies between groups of agents no longer
occur, the policy search space can be decoupled into independent components for which a policy may
be found separately while their combination is still globally optimal.

\begin{definition}[Conditional Reward Independence] Given an execution sequence~$\pad_t$, two
agents~$i, j \in \agents$ are \emph{conditionally reward independent}, denoted $\cri(i,j,\pad_t)$,
if for all future states $\state_t, \state_{t+1} \in \states$ and every future joint action $\jact_t
\in \actions$:
\begin{align}
 	\forall \depreward \in \rewards \; s.t.\ \; \{i,j\} \subseteq e\!:
 	\sum_{x=t}^{h-1}\reward^e( \state_x, \jact_x, \state_{x+1} ) = 0.
\end{align}
\vspace{-3mm}
\label{def:reward-indep}
\end{definition}

Although reward independence is concluded from joint execution sequence~$\pad_t$, some
independence can be detected from the local execution sequence~$\pad^i_t$, for example when
agent~$i$ completes its dependent actions. This \emph{local conditional reward independence} occurs
when $\forall j \in N\!: \cri(i,j, \pad_t^i )$ and is easily detected from the state during CRG
generation. For each such state~$\state^i$, we find optimal policy~$\policy^*_i( \state^ i )$ and
add only the optimal transitions dictated by that policy to our CRG, further reducing the CRG
size.

\paragraph{Conditional Return Policy Search}
All the previous combined leads to the \emph{Conditional Return Policy Search (CoRe)}
(Algorithm~\ref{alg:core}).
CoRe performs a branch-and-bound search over the joint policy space, represented as a DAG with nodes
$\state_t$ and edges $\langle\jact_t,\newstate_{t+1}\rangle$, such that finding a joint policy
corresponds to selecting a subset of action arcs from the CRGs (corresponding to $\jact_t$ and
$\newstate_{t+1}$). First, however, the CRGs~$\crg_i$ are constructed for the local
rewards~$\rewards_i$ of each agent~$i \in \agents$, assigned heuristically to obtain the CRGs.
Preliminary experiments provided evidence that a balanced distribution of rewards over the CRGs
leads to the best results in the MPP domain. Further research is required to find effective
heuristics for other domains. The generation of the CRGs follows Definition~\ref{def:crg} using a
recursive procedure, during which we store bounds (Equation~\ref{eq:UB}) and flag local states that
are locally conditionally reward independent according to Definition~\ref{def:reward-indep}.
During the subsequent policy search CoRe detects when subsets of agents, $N' \subset N$, become
conditionally reward independent, and recurses on these subsets separately.

\def\currseq{\ensuremath{\pad^{\agents}_t}}
\begin{algorithm}[t] \smaller \SetAlgoLined \LinesNumbered

\KwIn{ CRGs $\Phi$, current execution sequence $\currseq$, planning horizon $h$,
agent (sub) set~$\agents$ }
\color{black}
\lIf{t = h}{\Return{$0$}\label{alg:core-terminal}}
$V^{*} \leftarrow 0$ \\
\ForEach{conditionally independent subset $\subagents \subseteq \agents$ given $\currseq$
\label{alg:core-indepsolve}}  {
 	\tcp{Compute weighted sums of bounds:}
	
	$\forall \jact_t^{N'}: \ub(\state_{\pad,t}^{N'},\jact_t^{N'})\leftarrow\sum_{\state_{t+1}^{N'}
	}\limits\transprob( \state_{\pad,t}^{N'}, \jact_t^{N'}, \state_{t+1}^{N'} )
	\ub(\state_{\pad,t}^{N'},\jact_t^{N'},\state_{t+1}^{N'})$ 
        \label{alg:core-upper}\\
	$\lb_{max}\leftarrow\max_{\jact_t}^{N'} \sum_{\state_{t+1}^{N'}}\limits
	\transprob( \state_{\pad,t}^{N'}, \jact_t^{N'}, \state_{t+1}^{N'} ) \lb(\state_{\pad,t}^{N'},\jact_t^{N'},\state_{t+1}^{N'})$\\
	\label{alg:core-lower}
 	\tcp{Find joint action maximising expected reward}
	\ForEach{$\jact_t^{N'}$ for which $\ub(\state_{\pad,t}^{N'},\jact_t^{N'}) \geq \lb_{max}$ \label{alg:core-prune}
	 } {  $V_{ \jact_t^{N'} } \leftarrow 0$ \\
		\ForEach{$\state_{t+1}^{N'}$ reachable from  $\state_{\pad,t}^{N'}$ and $\jact_t^{N'}$} {
			$V_{ \jact_t^{N'} } += \transprob( \state_{\pad,t}^{N'}, \jact_t^{N'}, \state_{t+1}^{N'} )
			\Big( R(\state_{\pad,t}^{N'}, \jact_t^{N'}, \state_{t+1}^{N'}) + 
			\mathtt{CoRe}( \Phi, {\currseq}' \append [\jact_t^{N'}, \state_{t+1}^{N'}], h,
			\subagents ) \; \Big)$\label{alg:core-recur}
		}

	$\lb_{max} \leftarrow \max( V_{ \jact_t^{N'} }, \lb_{max} )$ \tcp{update
	lb}
	\label{alg:core-updatelb} }
	$V^{*} += \max_{\jact_t^{N'}} V_{ \jact_t^{N'} }$ \label{alg:core-optval}\\
}
\Return{$V^{*}$ \label{alg:core-combine}}
\caption{${\tt CoRe}(\Phi,\currseq,h,\agents)$}
\label{alg:core}
\end{algorithm}

After construction of the CRGs, CoRe performs depth-first policy search
(Algorithm~\ref{alg:core}) over the (disjoint sub-)set of agents~$\agents$ with potential reward
interactions (line~\ref{alg:core-indepsolve}). These subsets are found with a connected component
algorithm on a graph with nodes~$\agents$ and an edge $(i,j)$ for every pair of agents~$i,j \in
\agents'$ that are still dependent given the current execution sequence~$\currseq$, or $\neg \cri(i,j,\currseq)$. In
lines~\ref{alg:core-indepsolve} to \ref{alg:core-optval}, we thus only consider local state
space~$\states^{\subagents} \subseteq \states$ and joint actions~$\jact \in
\actions^{\subagents}$\!\!. Lines~\ref{alg:core-upper} and \ref{alg:core-lower} determine the upper
and lower bounds for this subset of agents, retrieved from the CRGs, used to prune in
line~\ref{alg:core-prune}. For the remaining joint actions, CoRe recursively computes the expected
value by extending the current execution sequence with the joint action~$\jact_t$ and all possible
result states~$\state_{t+1}$ (line~\ref{alg:core-recur}), of which the highest is returned
(line~\ref{alg:core-optval}).
As an extra, the lower bound is tightened (if possible) after every evaluation
(line~\ref{alg:core-updatelb}).

\begin{theorem}[CoRe Correctness] Given TI-MMDP $\mmdp = \tuple{\agents, \states, \actions,
\transprob, \rewards}$ with (implicit) initial state~$\state_0$, CoRe always returns the optimal
MMDP policy value $\ptvalue^{*}( \state_0 )$ (Eq.~\ref{eq:optval}).
\label{thm:core}
\end{theorem}
\begin{proof} (Proof Sketch) Conditional reward independence enables optimal decoupling of policy
search, the bounds are admissible with respect to the optimal policy value and our pruning does not
exclude optimal execution sequences. The full proof can be found in the Appendix.
\end{proof}

\section{CoRe Example}
We present a two-agent example problem in which both agents have actions $a, b$ and $c$, but every
action can be performed only once within a 2~step horizon. Action $c^2$ of agent~$2$ is (for ease of
exposition) the only stochastic action with outcomes $c$ and $c'$, and corresponding probabilities
$0.75$ and $0.25$. There is only one interaction, between actions~$a^1$ and $a^2$, and the reward
depends on feature~$\feature^1$ of agent~$1$ being set from $\feature^1?$ to $\feature^1$ or $\neg
\feature^1$. Thus we have one interaction reward function with rewards $\reward^{1,2}(\feature^1?,
\{a^1, a^2 \}, \feature^1 )$ and $\reward^{1,2}(\feature^1?, \{a^1, a^2 \}, \neg \feature^1 )$, and
local rewards $\reward^{1}$ and $\reward^{2}$. Without specifying actual rewards, we
demonstrate the CRGs and CoRe.

\begin{figure}[t] \centering \includegraphics[width=0.9\textwidth]{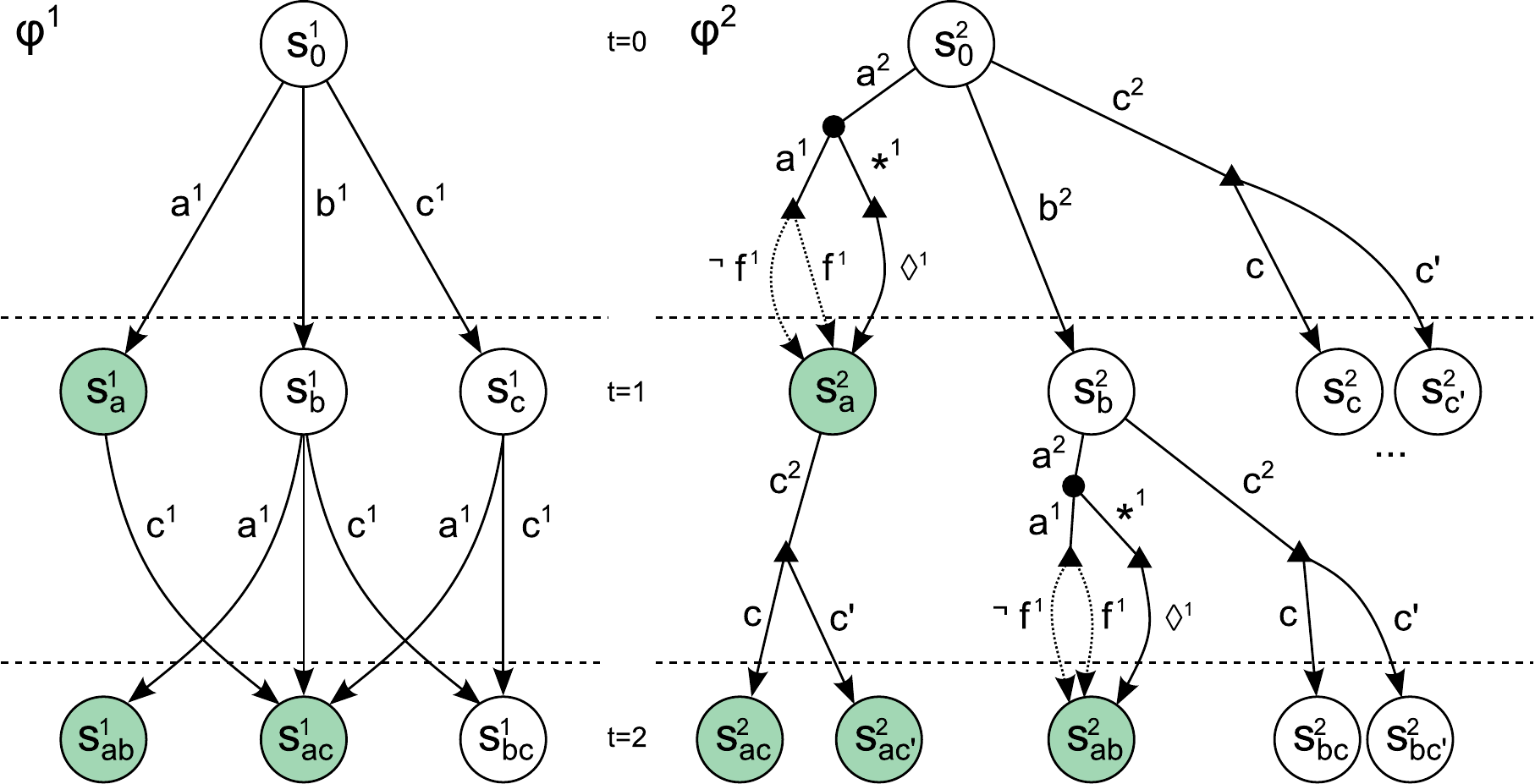} \caption{The CRGs of
the two agents. We omit the branches for $a^2$ and $b^2$ from states $\state^2_c$ and
$\state_{c'}^2$. The highlighted states are locally reward independent (reward arcs are omitted).}
	\label{fig:CRG}
\end{figure} 

Figure~\ref{fig:CRG} illustrates the two CRGs. On the left is the CRG~$\crg^1$ of agent~$1$ with
only its local reward~$\reward^1$, while the CRG of agent~2 includes both the reward interaction
function~$\reward^{1, 2}$ and its local reward~$\reward^2$. Notice that only when sequences start
with action~$a^2$ additional arcs are included in CRG~$\crg^2$ to account for reward interactions.
The sequence starting with~$a^2$ is followed by an after-state node with two arcs: one for agent~$1$
performing~$a^1$ and one for its other actions, $\wildcard^1 = \{ b^1, c^1\}$. The interaction
reward depends on what feature~$\feature^1$ is (stochastically) set to, thus the influence arcs
$\feature^1$ and $\neg \feature^1$ are now required. As the interaction reward only occurs when
$\{a^1, a^2\}$ is executed, the fully-specified after-state node after~$a^2$ and $\wildcard^1$ (the
triangle below it) has a no-influence arc~$\noinfl^1$. All other transitions are reward independent
and captured by local transitions $(\state^1_0, b^1, \state^1_b)$ and $(\state^1_0, c^1,
\state^1_c)$. Locally reward independent states are highlighted green and from each of these states
only the optimal action transition is kept in the CRG, e.g.\ only action arc~$c^1$ is included from
$\state_a^1$. This action was determined optimal from the local state by single-agent optimal policy
search, and thus no arcs for other actions (here $b^1$) are necessary.

Consider for example the state~$\state_{a}^1$ of $\crg^1$, in which action~$a^1$ has been taken in
the first time step. From this state onward, the reward interaction between action~$a^1$ and $a^2$
will no longer take place and therefore agent~$1$ is locally independent.
Consequentially, we can remove either the branch for action~$b^1$ or $c^1$ based on which action
maximises the expected value. In addition, this action will cause both agents to become independent
from each other because of the single dependency function between them.

\begin{figure}[t]
	\centering
	\includegraphics[width=0.9\textwidth]{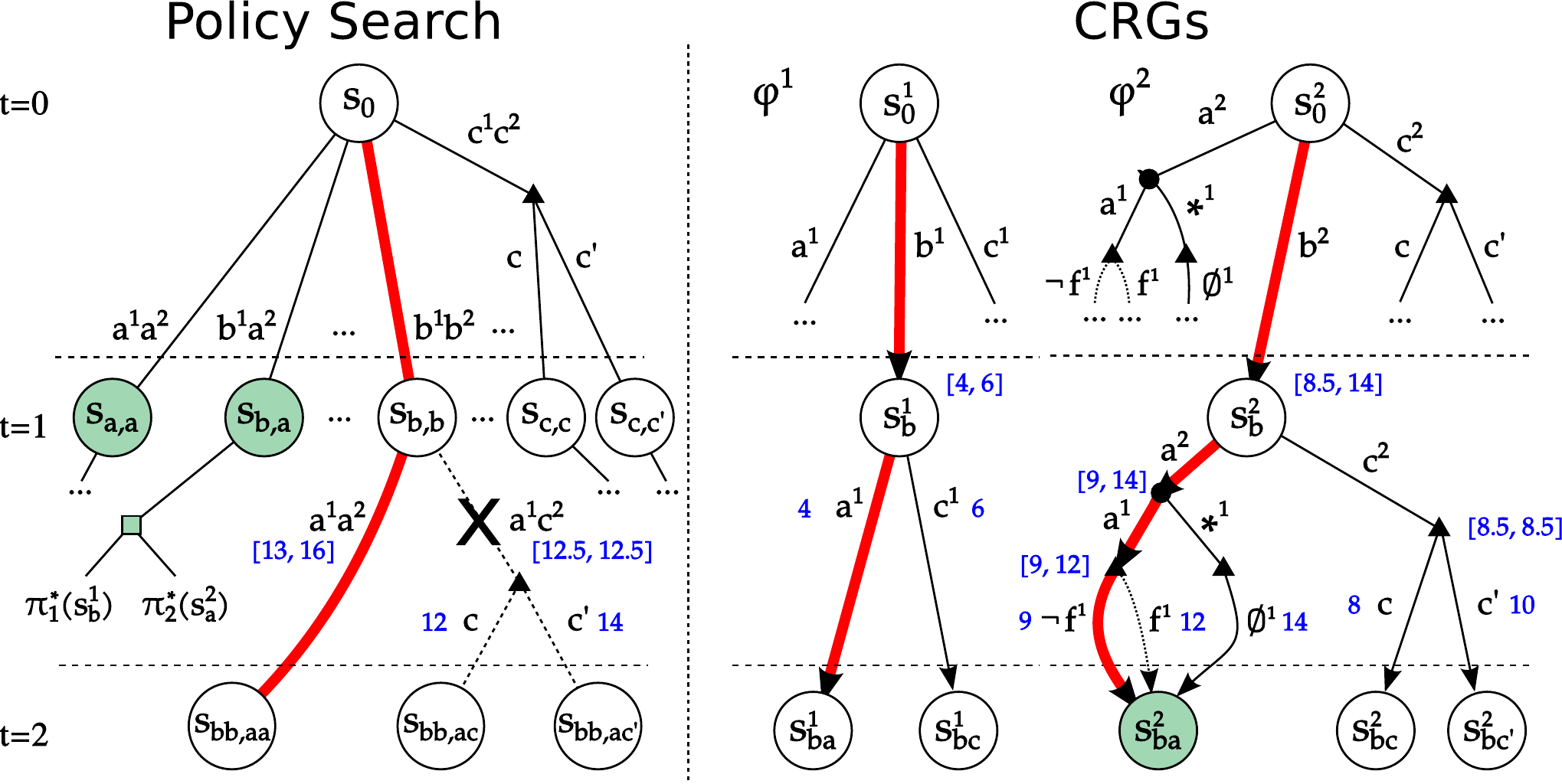}
	\caption{Example of policy evaluation. The left graph
	shows (a part of) the policy search tree with joint states and joint actions, and the right graph
	the	CRGs per agent. One possible execution sequence $\pad_h$ is highlighted in thick red.}
	\label{fig:core}
\end{figure}

\paragraph{Policy Search}
An example of CoRe policy search is shown in Figure~\ref{fig:core}, with the policy search space on
the left and the CRGs on the right, now annotated with return bounds. Only several of the branches
of the full DAG and CRGs are shown to preserve clarity. At $t=0$, there are 9 joint actions with 12
result states, while the CRGs need only 3 + 4 states and 3 + 6 transitions to represent all rewards.
The execution sequence $\pad_h$ that is evaluated is highlighted in thick red. This sequence starts
with non-dependent actions~$\{ b^1, b^2 \}$, resulting in joint state~$\state_{b,b}$ (ignore the
bounds in blue for now). The execution sequence at $t=1$ is thus $\pad_1 = [\state_0, \{ b^1, b^2
\}, \state_{b,b}]$. In the CRGs the corresponding transitions to states $\state^1_b$ and
$\state^2_b$ are shown. Now for $t=1$ CoRe is evaluating joint action~$\{a^1,a^2\}$ that is
reward-interacting and thus the value of state feature~$\feature^1$ is required to determine the
transition in $\crg^2$ (here chosen arbitrarily as $\neg \feature^1_x$). The corresponding execution
sequence (of agent 2) is therefore~$\pad^2_2 = [ \state^2_0, \{ b^1, b^2 \}, \state^2_b \cup
\{\feature^1?\}, \{ a^1, a^2 \}, \state^2_{ba} \cup \{\neg \feature^1\} ]$ If agent~1 had chosen
action~$c^1$ instead, we would traverse the branches~$\wildcard^1$ and $\noinfl^1$ leading to
state~$\state^2_{ba}$ without reward interactions.

Branch-and-bound is shown (in blue) for state~$\state_{b,b}$, with the rewards labelled on
transitions and their bounds at the nodes. The bounds for joint actions $\{ a^1, a^2 \}$ and $\{
a^1, c^2 \}$ are $[13, 16]$ and $[12.5, 12.5]$, respectively, found by summing the CRG bounds, hence
$\{a^1, c^2\}$ can be pruned. Note that we can compute the expected value of $\{ a^1, c^2 \}$ in the
CRG, but not that of $\{ a^1, a^2\}$. This is because agent~$2$ knows the transition probability of
action~$c^2$ but it does not know what value $\feature^1$ has during CRG generation or with what
probability~$a^1$ will be performed. Regardless, we can bound the return of action~$a^2$ over all
possible feature values, stored in $\crg^2$, and they can be updated as the probabilities become
known during policy search.

Conditional reward independence occurs in the green states of the policy search tree. After joint
action $\{ b^1, a^2 \}$, the agents will no longer interact ($a^2$ is completed and will not be
available anymore) and thus the problem is decoupled. From state~$\state_{b,a}$ CoRe finds optimal
policies $\policy_1^*( \state^1_b)$ and $\policy_2^*(\state^2_a)$ and combines them into the optimal
joint policy~$\policy^*( \state_{b,a}) = \langle \policy_1^*( \state^1_b), \policy_2^*(\state^2_a)
\rangle$ for the planning problem remaining from independent state~$\state_{b,a}$.

\section{Evaluation}
In our experiments we find optimal policies for the \emph{maintenance planning problem} (MPP, see
the introduction) that minimise the (time-dependent) maintenance costs and economic losses due to
traffic hindrance. In this problem agents represent contractors that participate in a mechanism and
thus it is essential that the planning is done optimally.

The problem can be modelled as an MDP, as is explained in full detail in \cite{Scharpff2013mpp}.
Here we only briefly outline the problem. Agents maintain a state with start and end times of their
maintenance tasks. Each of these tasks can be performed only once and agents can perform exactly one
task at a time (or do nothing). The individual rewards are given by maintenance costs that are task
and time dependent, while interaction rewards model network hindrance due to concurrent maintenance.
Maintenance costs are task and time dependent, while interaction rewards model network hindrance due
to concurrent maintenance. In this domain we conduct three experiments with CoRe to study 1) the
expected value when solving centrally versus decentralised methods, 2) the impact on the number of
joint actions evaluated and 3) the scalability in terms of agents.

\begin{figure}[tp] \centering
	\providecommand{\sfwidth}{0.48}
	\providecommand{\figwidth}{0.98\textwidth}
	\begin{subfigure}[b]{\sfwidth\textwidth}
		\centering
		\includegraphics[width=\figwidth]{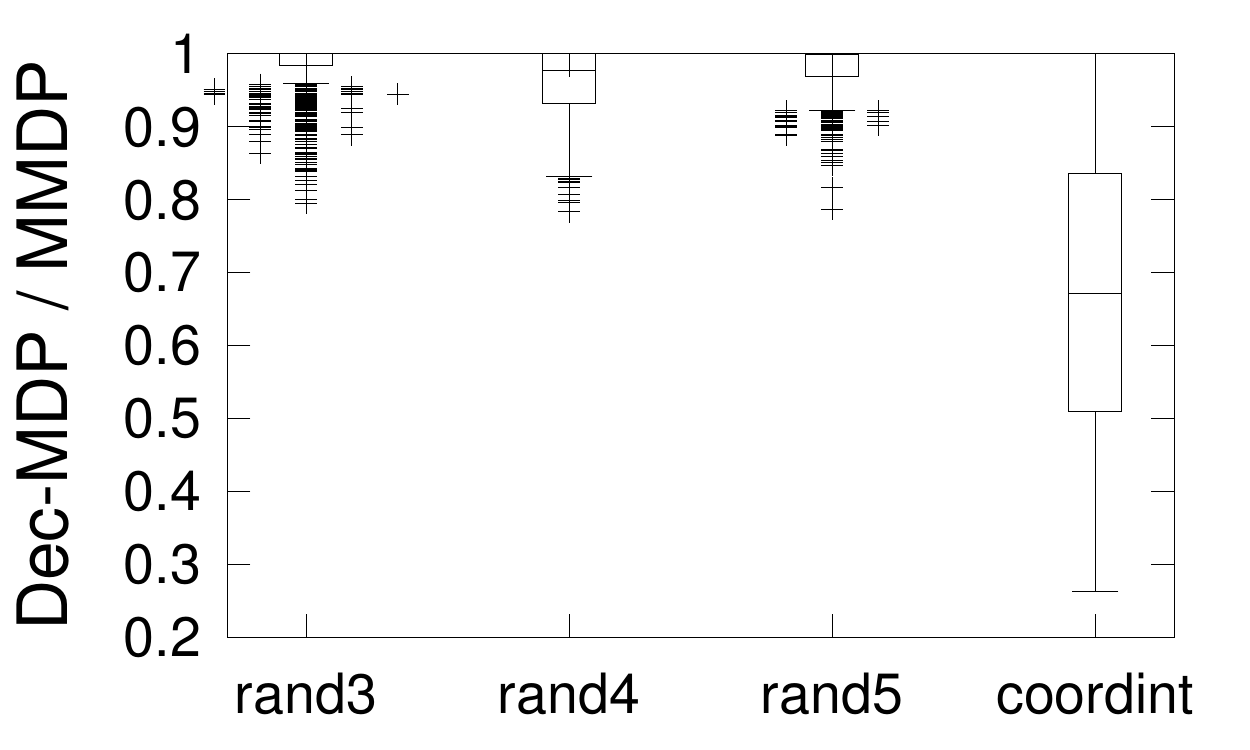}
		\caption{$\ptvalue^{\policy^*}_{\mathit{DEC}} / \ptvalue^{\policy^*}_{\mathit{MMDP}}$}
		\label{fig:decreward}
	\end{subfigure}
	\begin{subfigure}[b]{\sfwidth\textwidth}
		\centering
		\includegraphics[width=\figwidth]{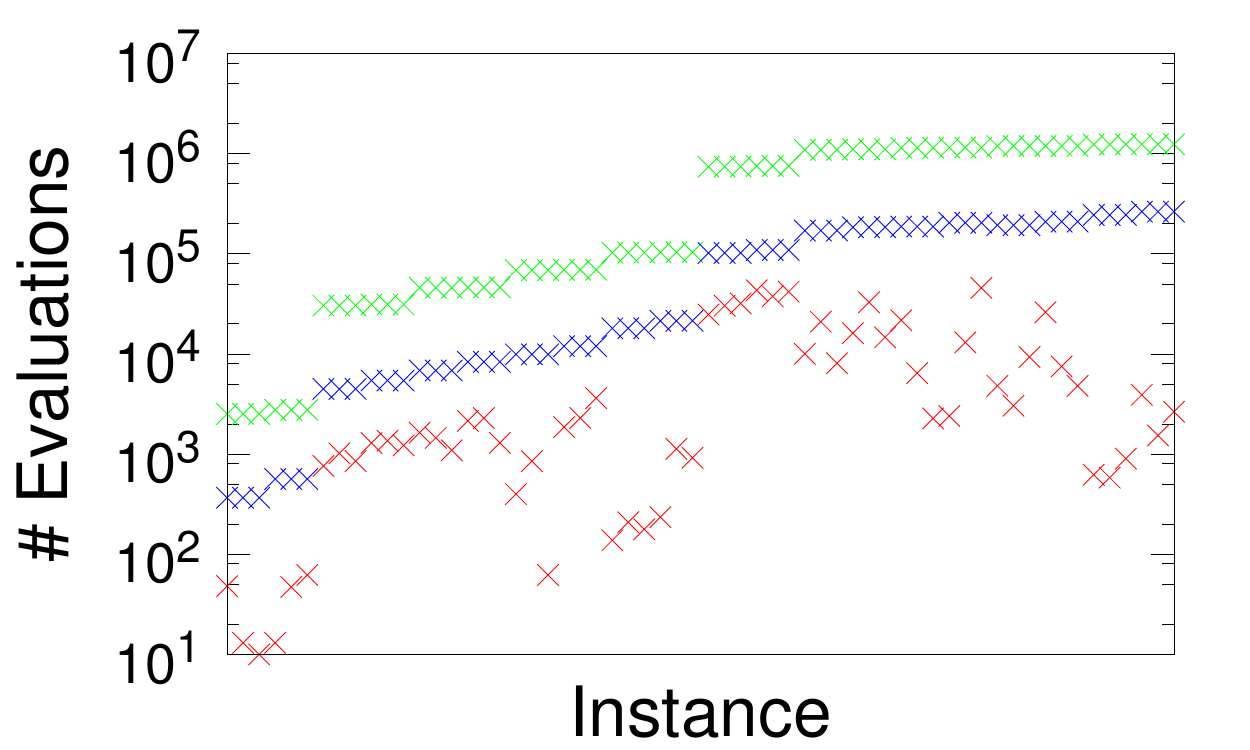}
		\caption{\# joint actions evaluated}
		\label{fig:jactions}
	\end{subfigure}
	\begin{subfigure}[b]{\sfwidth\textwidth}
		\centering
		\includegraphics[width=\figwidth]{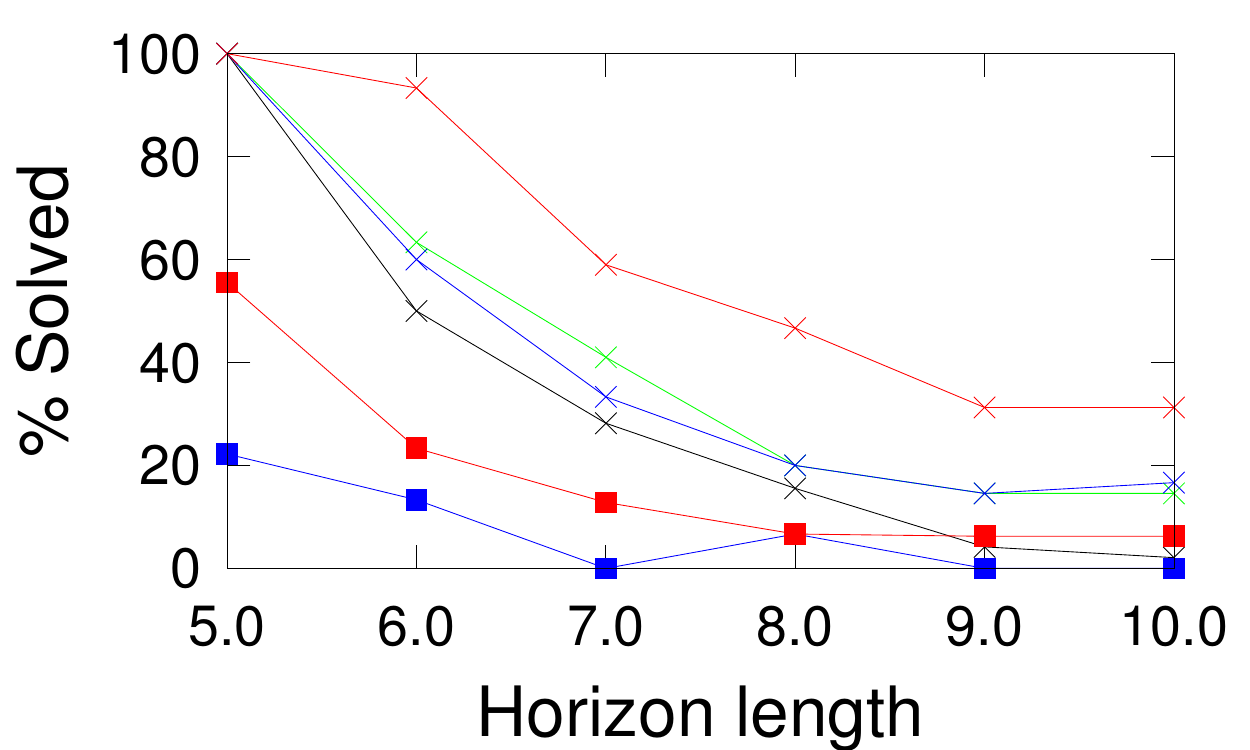}
		\caption{Perc.\ instances solved}
		\label{fig:numsolved}
	\end{subfigure}
	\begin{subfigure}[b]{\sfwidth\textwidth}
		\centering
		\includegraphics[width=\figwidth]{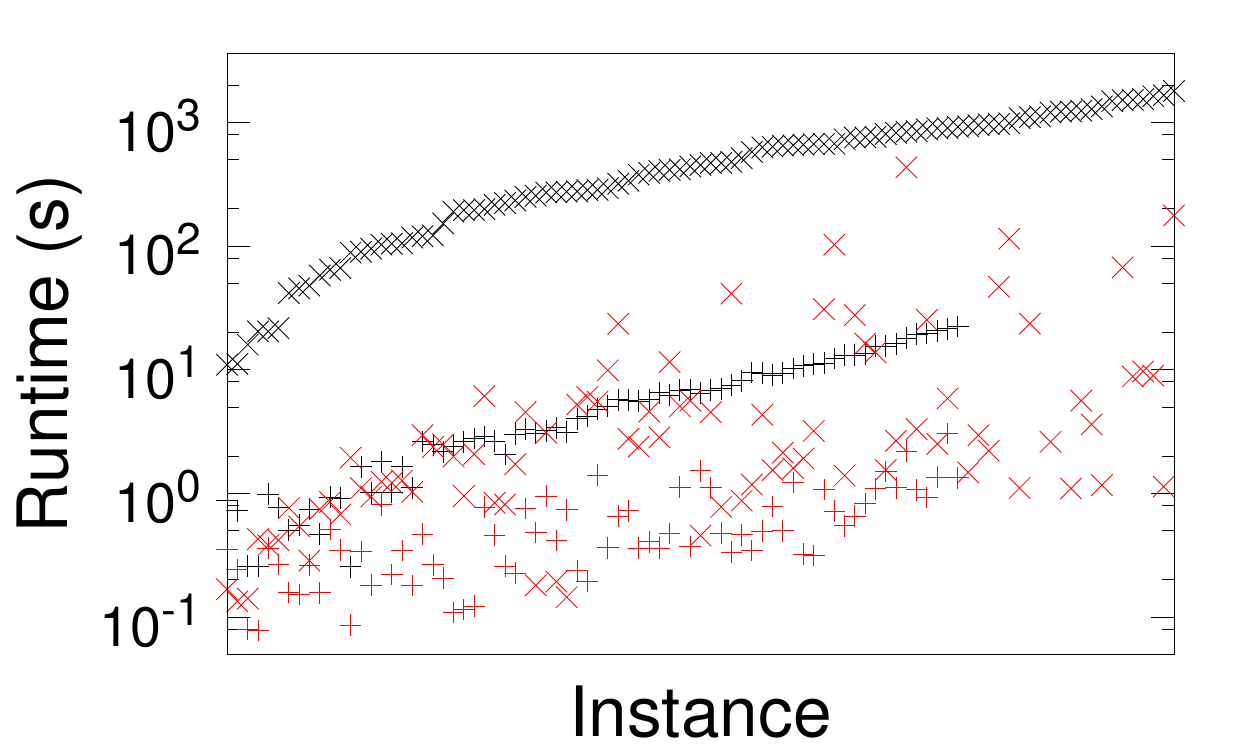}
		\caption{runtimes}
		\label{fig:runtime}	
	\end{subfigure}
	\begin{subfigure}[b]{\sfwidth\textwidth}
		\centering
		\includegraphics[width=\figwidth]{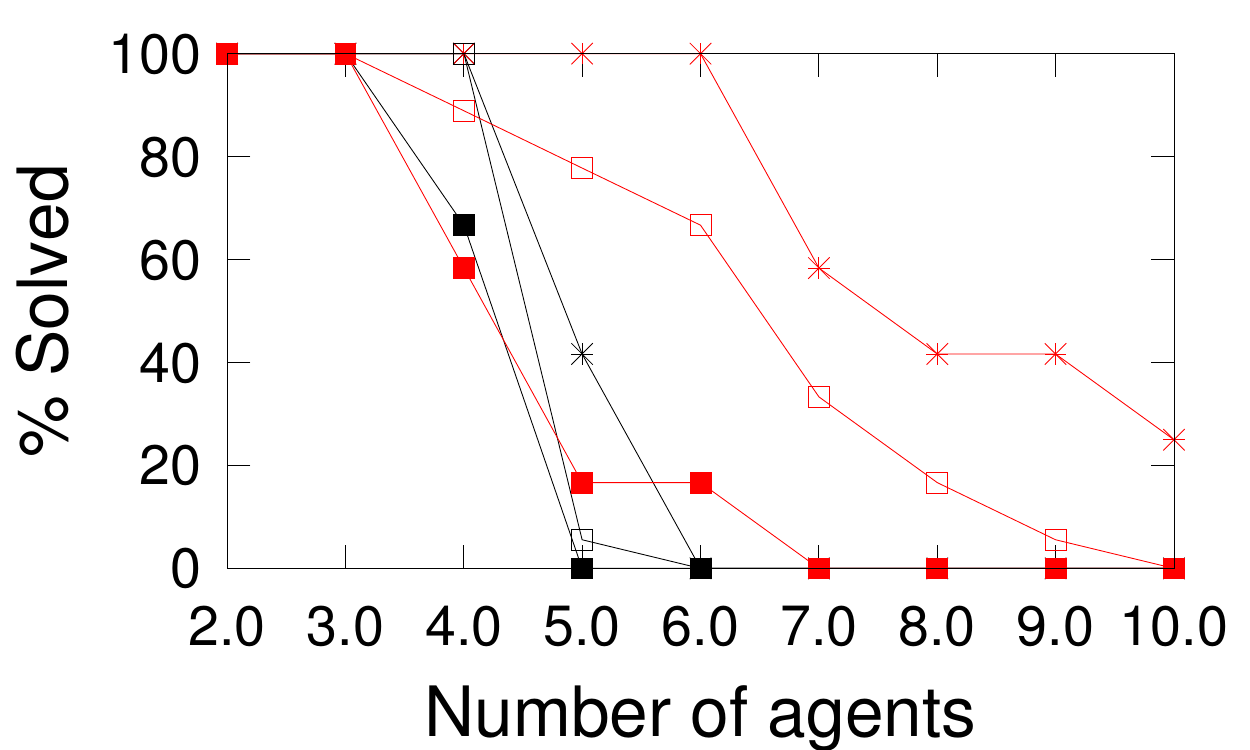}
		\caption{Perc.\ instances solved}
		\label{fig:pyra}	
	\end{subfigure}
	\begin{subfigure}[b]{\sfwidth\textwidth}
		\centering
		\includegraphics[width=0.9\textwidth]{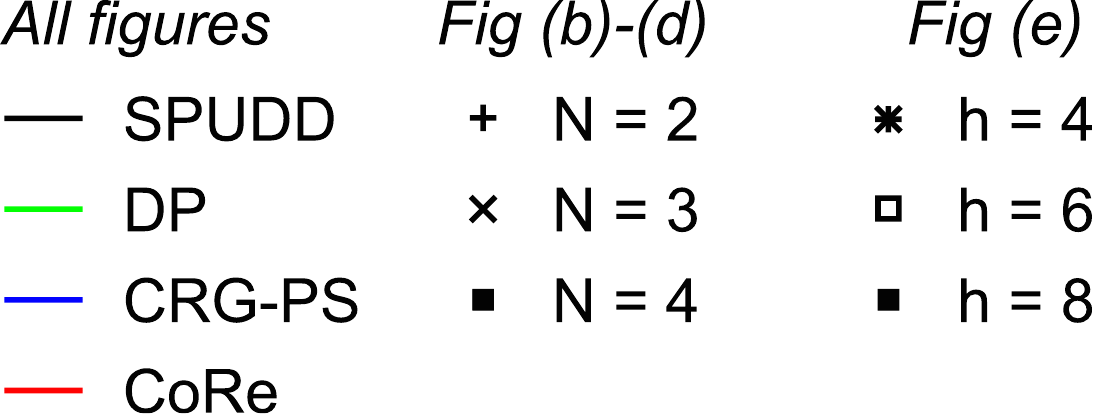}
		\vspace{13mm}
	\end{subfigure}
\caption{Experimental results: (\subref{fig:decreward}) The ratio of expected reward of the optimal
Dec-MDP policy versus the expected reward of the optimal MMDP policy for 1000 2-agent, 2-activity
instances of four sets of problems (3 sets of random problems {\tt rand[h]} with horizon 3, 4 and 5
and a set of coordination intensive instances {\tt coordint}), (\subref{fig:jactions}) Number of
joint actions evaluated by the dynamic programming algorithm (DP), CRG-enabled policy search (CRG-PS) and CoRe for the {\tt mpp} instances (log scale),
(\subref{fig:numsolved}) percentage of instances from {\tt mpp} that have been solved by each of the
algorithms within 30 minutes, (\subref{fig:runtime}) runtime comparison of SPUDD and CoRe for all
{\tt mpp} instances that were successfully solved by SPUDD (also log scale) and (\subref{fig:pyra})
the percentage of {\tt pyra} instances solved within the 30 minute time limit.}
	\label{fig:results}
\end{figure}
 
First, we compare with a decentralised baseline by treating the problem as a (transition and
observation independent) Dec-MDP \cite{Becker2003AAMAS} in which agents can only observe their
local state. Although the (TI-)Dec-MDP model is fundamentally different from the TI-MMDP -- in the
latter decisions are coordinated on \emph{joint} (i.e., global) observations -- the advances in
Dec-MDP solution methods \cite{Dibangoye13} may be useful for TI-MMDP problems if they can deliver
sufficient quality policies. That is, since they assume less information available, the value of
Dec-MDP policies will \emph{at best} equal that of their MMDP counterparts, but in practice the
expected value obtained from following a decentralised policy may be lower. We investigate if this
is the case in our first experiment, which compares the expected value of optimal MMDP policies found
by CoRe with optimal Dec-MDP policies, as found by the GMAA-ICE* algorithm \cite{Oliehoek13JAIR}. 

For this experiment we use two benchmark sets: {\tt rand[h]}, 3 sets of 1000 random two-agent
problems with horizons $h \in [3,4,5]$, and {\tt coordint}, a set of 1000 coordination-intensive
instances. The latter set contains tightly-coupled agents with dependencies constructed in such a
way that maintenance delays inevitably lead to hindrance unless agents coordinate their decisions
when such a delay becomes known, which is ar the first time step after starting the maintenance
task. Figure~\ref{fig:decreward} shows the ratio $\ptvalue^{\policy^*}_{\mathit{DEC}} /
\ptvalue^{\policy^*}_{\mathit{MMDP}}$ of the expected value of the optimal Dec-MDP
policy~$\ptvalue^{\policy^*}_{\mathit{DEC}}$ versus that of the optimal MMDP
policy~$\ptvalue^{\policy^*}_{\mathit{MMDP}}$. In the random instances the expected values of both policies
equal in approximately half of the instances. For coordination-intensive instances {\tt coordint}
decentralised policies result in worse results -- on average the reward loss is about $33\%$, but it
can be $75\%$ -- demonstrating that decentralised policies are inadequate for our purposes.
As this experiment only served to establish that decentralised methods are indeed not applicable,
from now on only centralised methods are considered.

In our remaining experiments we used a random test set~{\tt mpp} with 2, 3 and 4-agent problems (400
each) with 3 maintenance tasks, horizons $5$ to $10$, random delay probabilities and binary reward
interactions. We compare CoRe against several other algorithms to investigate the performance
of the algorithm. The current state-of-the-art approach to solve MPP, presented in
\cite{Scharpff2013mpp}, uses the value iteration algorithm SPUDD \cite{Hoey1999} and solves an
efficient MDP encoding of the problem. SPUDD uses algebraic decision diagrams to compactly represent
all rewards and is in this sense somewhat similar to our work, however it does not implicitly
partition and decouple rewards. Besides the SPUDD solver we included a dynamic programming algorithm
(DP) that uses a depth-first approach to maximise the Bellman equation of Equation~\ref{eq:optval}.
In addition to basic value iteration we implemented some domain knowledge in this algorithm to quickly
identify and prune sub-optimal and infeasible branches during evaluation. Finally, to analyse the
impact that branch-and-bound can have in a task-based planning domain such as MPP we added also a
CRG-enabled policy search algorithm (CRG-PS), a variant of our CoRe algorithm that uses CRGs but not
branch-and-bound pruning.

Using these algorithms, we first study the impact of using CRGs on the number of joint
actions that need to be evaluated. SPUDD is not considered in this experiment because because it
does not report this information. Figure~\ref{fig:jactions} shows the search space size reduction
by CRGs in this domain. Our CRG-enabled algorithm (CRG-PS, blue) approximately decimates the number
of evaluated joint actions compared to the DP method (green). Furthermore, when value bounds are
used (CoRe, red), this number is reduced even more, although its effect varies per instance.

Having observed the policy search space reduction that CoRe can achieved, we are interested in the
scalability of the algorithm in terms of number of agents and planning horizon.
Figure~\ref{fig:numsolved} shows the percentage of problems from the {\tt mpp} test set that are
solved within 30 minutes per method (all two-agent instances were solved and hence omitted).
CoRe solves more instances than SPUDD (black) of the 3 agent problems (cross marks), and only CRG-PS
and CoRe solve 4-agent instances. This is because CRGs successfully exploit the conditional action
independence that decouples the agents \emph{for most of the planning decisions}. Only when reward
interactions may occur actions are coordinated, whereas SPUDD always coordinates every joint
decision. Notice also that the use of branch-and-bound enables CoRe to solve more instances,
compared to the CRG-enabled policy search.

Next we investigate the runtime that was required by CoRe versus that by the current best known
method based on SPUDD. As CoRe achieves a greater coverage than SPUDD, we compare runtimes only for
instances successfully solved by the latter (Figure~\ref{fig:runtime}). We order the instances on
their SPUDD runtime (causing the apparent dispersion in CoRe runtimes) and plot runtimes of both.
Note that as a result the horizontal axis is not informative, it is the vertical axis plotting the
runtime that we are interested in. CoRe solves almost all instances faster than SPUDD, both with 2
and 3 agents, and has a greater solving coverage: CoRe failed to solve 3.4\% of the instances solved
by SPUDD whereas SPUDD failed 63.9\% of the instances that CoRe solved.

Finally, to study the agent-scalability of CoRe, we generated a test set~{\tt pyra} with a
pyramid-like reward interaction structure: every first action of the $k$-th agent depends on the
first action of agent~$2k$ and agent~$2k + 1$. Figure~\ref{fig:pyra} shows the percentage of solved
instances from the {\tt pyra} test for various problem horizons. Whereas previous state-of-the-art
solved instances up to only 5 agents, CoRe successfully solved about a quarter of the 10 agent
problems ($h=4$) and overall solves many of the previously unsolvable instances.

\section{Conclusions and Future Work}
In this work, we focus on optimally (and centrally) solving fully-observable, stochastic planning
problems where agents are dependent only through interaction rewards.
We partition individual and interaction rewards per agent in \emph{conditional return graphs}, a
compact and efficient data structure when interactions are sparse and/or non-recurrent. We propose a
conditional return policy search algorithm (CoRe) that uses reward bounds based on CRGs to reduce
the search space size, shown to be by orders of magnitude in the maintenance planning domain.  This
enables CoRe to overall decrease the runtime required compared to the previously best approach and
solve instances previously deemed unsolvable. The reduction in search space follows from three key
insights: 1) when interactions are sparse, the number of unique returns per agent is relatively
small and can be stored efficiently, 2) we can use CRGs to maintain bounds on the return, and thus
the expected value, and use this to guide our search, and 3) in the presence of conditional reward
independence, i.e.\ the absence of further reward interactions, we can decouple agents during policy
search.

Our experiments show that the impact of reduction is by orders of magnitude in the maintenance
planning domain. This enables CoRe to solve instances that were previously deemed unsolvable. In addition,
to scaling to larger instances,
CoRe almost always produces solutions faster than the previously best approach.
Moreover, CoRe is able to scale up to 10-agent instances when the reward structure exhibits a high
level of conditional reward independence, whereas previous methods did not scale beyond 5 agents.
Finally, our experiments also illustrate that using a decentralised MDP approach, a line of
research that has seen many scalable approaches in terms of agents and reward structures, leads to
suboptimal expected policy values.

Here only optimal solutions are considered, but CRGs can be combined with approximation in several
ways. First, the reward structure of the problem itself may be approximated. For
instance, the reward-function approximation of \cite{koller1999computing} can be applied to increase
reward sparsity, or CRG paths with relatively small reward differences may be grouped, trading off a
(bounded) reward loss for compactness. Secondly, the CRG bounds directly lead to a
bounded-approximation variant of CoRe, usable in for instance the approximate multi-objective method
of \cite{roijers2014bounded}. Lastly, the CRG structure can be implemented in any
(approximate) TI-MMDP algorithm or, vice versa, any existing approximation scheme for MMDP that
preserves TI can be used within CoRe.

Although we focused on transition-independent MMDPs, CRGs may be interesting for general MMDPs when
transition dependencies are sparse. This would require including dependent-state transitions in the
CRGs similar to reward-interaction paths and is considered to be future work. Another interesting
avenue is to exploit conditional reward independence during joint action generation.

\subsection*{Acknowledgements}
This research is supported by the NWO DTC-NCAP (\#612.001.109), Next Generation 
Infrastructures, Almende BV and NWO VENI (\#639.021.336) projects.

\appendix
\renewcommand{\theequation}{A.\arabic{equation}}
\renewcommand*{\thetheorem}{A.\arabic{theorem}}
\renewcommand*{\thelemma}{A.\arabic{lemma}}
\renewcommand*{\thecorollary}{A.\arabic{corollary}}
\newcommand{\theoremCore}{\ref{thm:core}}
\newcommand{\defRewardIndep}{\ref{def:reward-indep}}
\newcommand{\eqReturnDistr}{\ref{eq:returnDistr}}
\newcommand{\eqOptreturn}{\ref{eq:optreturn}}

\section*{Appendix: Proof of Theorem \theoremCore}
\label{apx:correctness}

In this appendix, we prove the correctness of the CoRe algorithm (Theorem \theoremCore).

We define several notational shorthands for convenience. For two (sub)sets of agents $A, B \subseteq
\agents$, $\deprewards_{AB} \subseteq \rewards$ is the set of all rewards for which $A \cap B
\cap e \neq \emptyset$. $\deprewards_{A\nin{B}} \subseteq \rewards$ is the set of rewards such
that $A \cap e \neq \emptyset$ and $B \cap e = \emptyset$. Observe that the individual rewards for
all agents~$a \in A$ are thus contained within $\deprewards_{A\nin{B}}$ (and similarly all rewards
$\reward^b$ are included in $\deprewards_{\nin{A}B}$). Furthermore, two agent sets~$A$ and $B$ are
conditionally reward independent, denoted $\cri( A, B, \pad^t)$, as result of history~$\pad_t$ iff
$\forall a \in A, b \in B\!: \cri(a,b,\pad_t)$. Finally, $\tr^e = (\{ \state^{j} \}_{j \in e},$ $\{
\jact^{j} \}_{j \in e}, \{ \newstate^j \}_{ j \in e })$ denotes a transition local to agents~$e$ and
a global transition, i.e.\ $e = \agents$, is denoted by $\tr$.

\begin{lemma}
Given an execution history~$\pad_t = [ \state_0, \jact_0, \ldots, \state_u, \ldots, \state_t ]$ up
to time~$t$ that can be partitioned into two histories, $\pad_u = [ \state_0, \ldots, \state_u ]$
and $\pad_{u'} = [ \state_u, \ldots, \state_t ] $, and a disjoint partitioning of agents~$\agents =
\agents_1 \cup \agents_2 \cup \ldots \cup \agents_k$ such that for every pair~$\agents_a, \agents_b
\in \agents$ it holds that $\cri( \agents_a, \agents_b, \pad_u )$ when $a \neq b$. The
return~$\mdpreturn( \pad_t )$ (as defined in Equation \eqReturnDistr) can be decoupled as:
\begin{equation}
	\mdpreturn( \pad_u ) + \sum_{i=1}^{k} \mdpreturn_{\agents_i}( \pad^{\agents_i}_{u'} )
	\label{eq:cri-decoupling}
\end{equation}
where~$\pad^{\agents_i}_{u'}$ is the execution history only containing the states and actions of
the agents in the set~$\agents_i \subseteq \agents$, starting from time~$u$.
\label{lem:cri}
\end{lemma}

\begin{proof}
Recall from Definition~\defRewardIndep{} that two agents~$i,j \in \agents$ are $\cri$ iff
$\forall \depreward \in \rewards \; s.t.\ \; \{i,j\} \subseteq e\!: \sum_{x=t}^{h}\reward^e( \state_x,
\jact_x, \state_{x+1} ) = 0$ for every pair of states $\state_t, \state_{t+1}$ and all joint
actions~$\jact_t$, given execution history~$\pad_t$. Moreover, recall that the MMDP
rewards~$\rewards$ w.l.o.g.\ are structured as $\reward( \tr ) = \sum_{\depreward \in \rewards}
\depreward( \tr^e )$.

Let $A$ and $B$ be disjoint subsets of agents such that $A \cup B = \agents$ and let the reward
functions be accordingly partitioned as disjoint sets: $\rewards = \deprewards_{A\nin{B}} \cup
\deprewards_{\nin{A}B} \cup \depreward_{AB}$. Now assume that for a given execution history~$\pad_t
= \pad_u \cup \pad_{u'}$ we have $\cri( A, B, \pad_u)$. From the state~$\state_u$ resulting from the
execution history~$\pad_u$, all future rewards can only be local with respect to subsets~$A$ and $B$
because every reward~$\depreward_{AB}$ must be zero by definition of CRI.
Therefore we can rewrite the (future) global reward~$\reward(\tr)$ of every possible
transition~$\tr$ (in every possible $\pad^{\agents_i}_{u'}$) as:
\begin{align}
	\sum_{\depreward \in \deprewards} \depreward( \tr ) & = \deprewards_{A\nin{B}}( \tr^{A} ) +
	\deprewards_{\nin{A}B}( \tr^{B} ) + \deprewards_{AB}( \tr^{AB} ) \\
	& = \sum_{\depreward \in \deprewards_{A\nin{B}}} \depreward( \tr^{A} ) +
	\sum_{\depreward \in \deprewards_{\nin{A}B}} \depreward( \tr^{B} )
\end{align}
where the transition decoupling is possible due to transitional independence. Remember that we can
write the returns for an execution history~$\pad_h$ as (Eq.~\eqReturnDistr)
\begin{align}
	\mdpreturn (\pad_t) &= \sum_{x=0}^{t-1} \sum_{\depreward\in \rewards} \depreward( \tr^e_{\pad,x} ) 
\end{align}
in which~$\tr^e_{\pad}$ denotes the transition in the execution history~$\pad_t$ local to
agents~$e$. Then, for two disjoint agent subsets~$A \cup B =\agents$ that have $\cri(A,B,\pad_u)$ as
a result of  $\pad_u$:
\begin{align}
	\mdpreturn(\pad_t) &= \mdpreturn(\pad_u) + \mdpreturn(\pad_{u'}) \\
	& = \mdpreturn(\pad_u) + \sum_{x=u}^{t-1} \left[ \deprewards_{A\nin{B}}( \tr^A_{\pad,x} ) +
	\deprewards_{\nin{A}B}( \tr^B_{\pad,x} ) \right] \\
	& = \mdpreturn(\pad_u) + \sum_{x=u}^{t-1} \deprewards_{A\nin{B}}( \tr^A_{\pad,x} ) +
	\sum_{x=u}^{t-1} \deprewards_{\nin{A}B}( \tr^B_{\pad,x} ) \\
	& = \mdpreturn(\pad_u) + \mdpreturn_A(\pad^A_{u'}) + \mdpreturn_B(\pad^B_{u'})
\end{align}

As a consequence, we can decouple the computation of returns for agent sets~$A$ and $B$ from
time~$u$.

For now we have only considered two agent sets~$A$ and $B$, however we can apply this
decoupling recursively, in order to obtain an arbitrary disjoint partitioning of agents such that $\agents_1 \cup \agents_2 \cup
\ldots \cup \agents_k = \agents$ and $\forall \agents_a, \agents_b \in \agents\!:
\cri(\agents_a, \agents_b,\pad_u)$. That is, without loss of generality, we choose $A = \agents_a$ and $B =
\agents_2 \cup \ldots \cup \agents_k$ and decouple the return as $\mdpreturn(\pad_u) +
\mdpreturn_A(\pad^A_{u'}) + \mdpreturn_B(\pad^B_{u'})$. We now observe that we can rewrite
$\mdpreturn_B$ itself as $\mdpreturn_{\agents_2}(\pad^{\agents_2}_{u'}) + \mdpreturn_{B \setminus
\agents_2}(\pad^{B \setminus \agents_2}_{u'})$ by following the same steps, because both sets again satisfy conditional reward independence. By continuing this process we obtain
Equation~\ref{eq:cri-decoupling}.
\qedhere
\end{proof}

As a result of Lemma~\ref{lem:cri}, Equation~\eqOptreturn{} and transitional independence we can now also decouple the policy values of two sets of agents, $\agents_a$ and $\agents_a$, from time $u$ onwards, when it holds that $\cri( \agents_a, \agents_b, \pad_u )$.
\begin{corollary} 
When at a timestep $u$, we have observed $\pad_u$ and there is a disjoint partitioning of agents~$\agents = \agents_1 \cup \agents_2 \cup \ldots \cup \agents_k$ such that for every pair~$\agents_a, \agents_b \in \agents$ it holds that $\cri( \agents_a, \agents_b, \pad_u )$ when $a \neq b$, the value of a given policy $\pi$,  $\ptvalue( \state_t )$ can be decoupled as: 
\begin{equation}
	\ptvalue^\pi( \state_t ) = \sum_{i=1}^k \ptvalue^\pi_{\agents_i}( \state_t^{\agents_i} ) 
\end{equation}
\label{cor:decoupling}
\end{corollary}
\begin{proof} 
Starting from  Equation~\eqOptreturn, we first observe that the return $\mdpreturn(\pad_{u'})$ from
timestep $u$ onwards is equal to $\sum_{i = 1}^k \mdpreturn_{\agents_i}( \pad^{\agents_i}_{u'} )$
(Lemma \ref{lem:cri}) and that, because of transition independence, each set $\agents_i$ of
agents has independent probability distributions over future execution histories $\pad_{u'}$. Thus
we have the following equalities:

\begin{align}
	\ptvalue( \state_t) &= \sum_{\pad_{u'} | \policy, \pad_t } \prob( \pad_{u'} ) \mdpreturn(
	\pad_{u'}) \\
	& = \sum_{\pad_{u'} | \policy, \pad_t } \prob(
	\pad_{u'} )  \sum_{i = 1}^k \mdpreturn_{\agents_i}( \pad^{\agents_i}_{u'} )
	= \sum_{\pad_{u'} | \policy, \pad_t } \sum_{i = 1}^k  \prob(
	\pad_{u'} )   \mdpreturn_{\agents_i}( \pad^{\agents_i}_{u'} ) \\
	& = \sum_{i = 1}^k \sum_{\pad_{u'} | \policy, \pad_t } \prob(
	\pad_{u'} )   \mdpreturn_{\agents_i}( \pad^{\agents_i}_{u'} )
	= \sum_{i = 1}^k  \sum_{\pad^{\agents_i}_{u'} | \policy, \pad_t } \prob(
	\pad^{\agents_i}_{u'} ) \mdpreturn_{\agents_i}( \pad^{\agents_i}_{u'} ) \\
	& = \sum_{i=1}^k \ptvalue^\pi_{\agents_i}( \state_t^{\agents_i} ) 
\end{align}
\end{proof}

\begin{lemma} The bounding heuristics $\lb( \state^e_t )$ and $\ub( \state^e_t )$ used to prune during
branch-and-bound search are admissible with respect to the expected value~$\ptvalue( \state^e_t )$ of
state~$\state^e_t$ for agents~$e \subseteq \agents$ at time~$t$.
\label{lem:bounds}
\end{lemma}
\begin{proof} 
We proof the admissibility of the bounding heuristics by induction. For sake of brevity, we
only show the upper bound proof, but the proof for the lower bound can be written down accordingly.
Recall that $\reward = \bigcup_{i \in \agents} \rewards_i$ is the disjoint partitioning of reward
functions over the CRGs.

First, consider the very last timestep, $h-1$, for which there is no future reward, i.e., 
\begin{align}
\ptvalue( \state_{h-1} ) &= \max_{\jact} \sum_{\state_{h} \in \states } \transprob( \tr_{h-1}) 
	\reward(\tr_{h-1} )
	= \max_{\jact} \sum_{\state_{h} \in \states } \transprob( \tr_{h-1}) \sum_{i \in \agents} 
	\rewards_i(\tr^e_{h-1} ) \\
	&\leq \max_{\jact} \max_{\state_{h} \in \states } \sum_{i \in \agents}  \rewards_i(\tr^e_{h-1} )\\
	&\leq \sum_{i \in \agents}  \max_{\tr^e_{h-1}\in\crg_i(\state^i_{h-1})} \rewards_i(\tr^e_{h-1} )\\
	&= \sum_{i \in \agents}  \ub(\state^i_{h-1}) 
\end{align}
Then, we show that if for a next stage $t+1$ we have a valid upper bound, the value for a state $\state_t$ is also upper bounded by $\sum_{i \in \agents} \ub(s^i_t)$. And therefore, that because  $\ub(\state^i_{h-1}) $ is a valid upper bound on $\ptvalue( \state_{h-1} )$, the upper bound is admissible for all stages before $h-1$: 
\begin{align}
	&\ptvalue( \state_t ) = \max_{\jact} \sum_{\state^e_{t+1} \in \states} \transprob( \tr_t
	) \left( \reward( \tr_t ) + \ptvalue( \state_{t+1}) \right) \\
	& = \max_{\jact^e} \sum_{\state_{t+1} \in \states } \transprob( \tr_t ) \left( 
	\ptvalue( \state_{t+1} ) + \sum_{i \in \agents} \rewards_i(\tr^e_{t} )  \right)\\
	& \le  \max_{\jact^e} \sum_{\state_{t+1} \in \states } \transprob( \tr_t ) \sum_{i \in \agents} \left( 
	 \rewards_i(\tr^e_{t} ) + \ub(\state^i_{t+1} )  \right)\\
	& \le  \sum_{i \in \agents} \max_{\tr^e_{t}\in\crg_i(\state^i_t)} \left( 
	 \rewards_i(\tr^e_{t} ) + \ub(\state^i_{t+1} )  \right)\\
	& = \sum_{i \in \agents} \ub(\state^i_{t} )  
\end{align}

From the bounds on the state-values,
\[
\lb(\state_{t} )  = \sum_{i \in \agents} \lb(\state^i_{t} ) \leq \ptvalue( \state_t )  \leq \sum_{i \in \agents} \ub(\state^i_{t} ) =  \ub(\state_{t} ),
\]
we can also distill admissible bounds on state-action values, $Q( \state_t, \jact )$. Taking the standard MDP definition for $Q( \state_t, \jact )$, 
\[
 Q( \state_t, \jact ) = \!\!\!\sum_{\newstate_{t+1} \in \states}\!\!\! \transprob( \state_t, \jact,
 \newstate_{t+1}) \left( \reward( \state_t, \jact, \newstate_{t+1}) +   \ptvalue( \state_{t+1} ) \right),
\] 
we replace $\ptvalue( \state_{t+1} )$ by the corresponding lower or upper bounds: 
\[
 B( \state_t, \jact ) = \!\!\!\sum_{\newstate_{t+1} \in \states}\!\!\! \transprob( \state_t, \jact,
 \newstate_{t+1}) \left( \reward( \state_t, \jact, \newstate_{t+1}) +   B( \state_{t+1} ) \right).
\] 
Using $B( \state_t, \jact )$, CoRe can exclude a joint action~$\jact$ after execution
history~$\pad_t$ from  consideration when there is another joint action~$\jact'$ for which $\ub( \state_t, \jact ) \leq \lb( \state_t, \jact' )$, as is standard in branch-and-bound algorithms.
\end{proof}

\begin{proof}[Concluding the proof of Theorem \theoremCore] As a direct consequence of
Corollary~\ref{cor:decoupling} agents can be decoupled during policy search without losing
optimality. Moreover, Lemma~\ref{lem:bounds} shows that both the upper and lower bounds are
admissible heuristic functions with respect to the expected policy value from a given state. In the
main loop, the CoRe algorithm recursively expands and evaluates all possible extensions to the
current execution path, except for those starting with actions that lead to a lower upper bound than
another action's lower bound.

As the policy search considers all possible execution histories, excluding pruned, non-optimal ones,
the search will eventually return the optimal policy value~$\ptvalue^*$, and corresponding policy,
thus proving Theorem \theoremCore.
\end{proof}

Moreover, as there is only a finite number of execution histories, the CoRe algorithm is also
guaranteed to terminate in a finite number of recursions.

\twocolumn
\bibliography{aaai16}
\bibliographystyle{apalike}

\end{document}